%% file: main.tex
\documentclass[sigconf]{acmart}
\def\BibTeX{{\rm B\kern-.05em{\sc i\kern-.025em b}\kern-.08emT\kern-.1667em\lower.7ex\hbox{E}\kern-.125emX}}
    
\usepackage{nicefrac}
\usepackage{siunitx}
\usepackage{array,framed}
\usepackage{booktabs}
\usepackage{
  color,
  float,
  epsfig,
  wrapfig,
  graphics,
  graphicx
}
\usepackage{textcomp}
\usepackage{setspace}
\usepackage{latexsym,fancyhdr,url}
\usepackage{enumerate}
\usepackage{algorithm2e}
\usepackage{algpseudocode}
\usepackage{graphics}
\usepackage{xparse} 
\usepackage{xspace}
\usepackage{multirow}
\usepackage{csvsimple}
\usepackage{balance}
\usepackage{hhline}
\usepackage{
  tikz,
  pgfplots,
  pgfplotstable
}

\usepackage{hyperref}

\usetikzlibrary{
  shapes.geometric,
  arrows,
  external,
  pgfplots.groupplots,
  matrix
}

\pgfplotsset{compat=1.9}

\usepackage{mathtools,}

\DeclareMathAlphabet{\mathcal}{OMS}{cmsy}{m}{n}

\DeclareGraphicsExtensions{%
    .png,.PNG,%
    .pdf,.PDF,%
    .jpg,.mps,.jpeg,.jbig2,.jb2,.JPG,.JPEG,.JBIG2,.JB2}

\input{defs}
\usepackage{subfig}
\usepackage[export]{adjustbox}

\usepackage{fancyvrb}
\usepackage{listings}
\input{macros}

\begin{document}
\fancyhead{}
\def\thetitle{Optimal Mixed Discrete-Continuous Planning \\ for  Linear Hybrid Systems}
\title{\thetitle}

\author{Jingkai Chen}
\affiliation{jkchen@csail.mit.edu \\
Massachusetts Institute of Technology}

\author{Brian C. Williams}
\affiliation{williams@csail.mit.edu \\
Massachusetts Institute of Technology}

\author{Chuchu Fan}
\affiliation{chuchu@mit.edu \\
Massachusetts Institute of Technology}

\begin{abstract}
Planning in hybrid systems with both discrete and continuous control variables is important for dealing with real-world applications such as extra-planetary exploration and multi-vehicle transportation systems. Meanwhile, generating high-quality solutions given certain hybrid planning specifications is crucial to building high-performance hybrid systems. However, since hybrid planning is challenging in general, most methods use greedy search that is guided by various heuristics, which is neither complete nor optimal and often falls into blind search towards an infinite-action plan. In this paper, we present a hybrid automaton planning formalism and propose an optimal approach that encodes this planning problem as a Mixed Integer Linear Program (MILP) by fixing the action number of automaton runs. We also show an extension of our approach for reasoning over temporally concurrent goals. By leveraging an efficient MILP optimizer, our method is able to generate provably optimal solutions for complex mixed discrete-continuous planning problems within a reasonable time. We use several case studies to demonstrate the extraordinary performance of our hybrid planning method and show that it outperforms a state-of-the-art hybrid planner, Scotty, in both efficiency and solution qualities.
\end{abstract}



\maketitle

\section{Introduction}

Hybrid systems are a powerful modeling framework to capture both the physical plants and embedded computing devices of complex cyber-physical systems. When planning the desired behaviors of a hybrid system, we have to consider both the discrete actions taken by the computing units and the continuous control inputs for the physical actuators. This poses unique and significant challenges in planning for hybrid systems, as one has to consider the change of dynamics of the continuous flow by the control inputs, the interleaves of continuous flows and the discrete transitions between modes, resets associated with transitions, and concurrently running agents in multi-agent systems. 

Planning mixed discrete-continuous strategies in hybrid systems is theoretically difficult: on the discrete side, planning with numeric state variables has been proved to be undecidable~\cite{helmert2002decidability}; on the continuous side, computing the exact unbounded time reachable states for hybrid systems is also a well-known undecidable problem~\cite{henzinger1998s}. Nevertheless we are interested in high-quality solutions with the shortest makespan or lowest energy consumption, which is crucial to designing high-performance systems.

\begin{figure}[t]
\centering
\includegraphics[width=0.95\columnwidth]{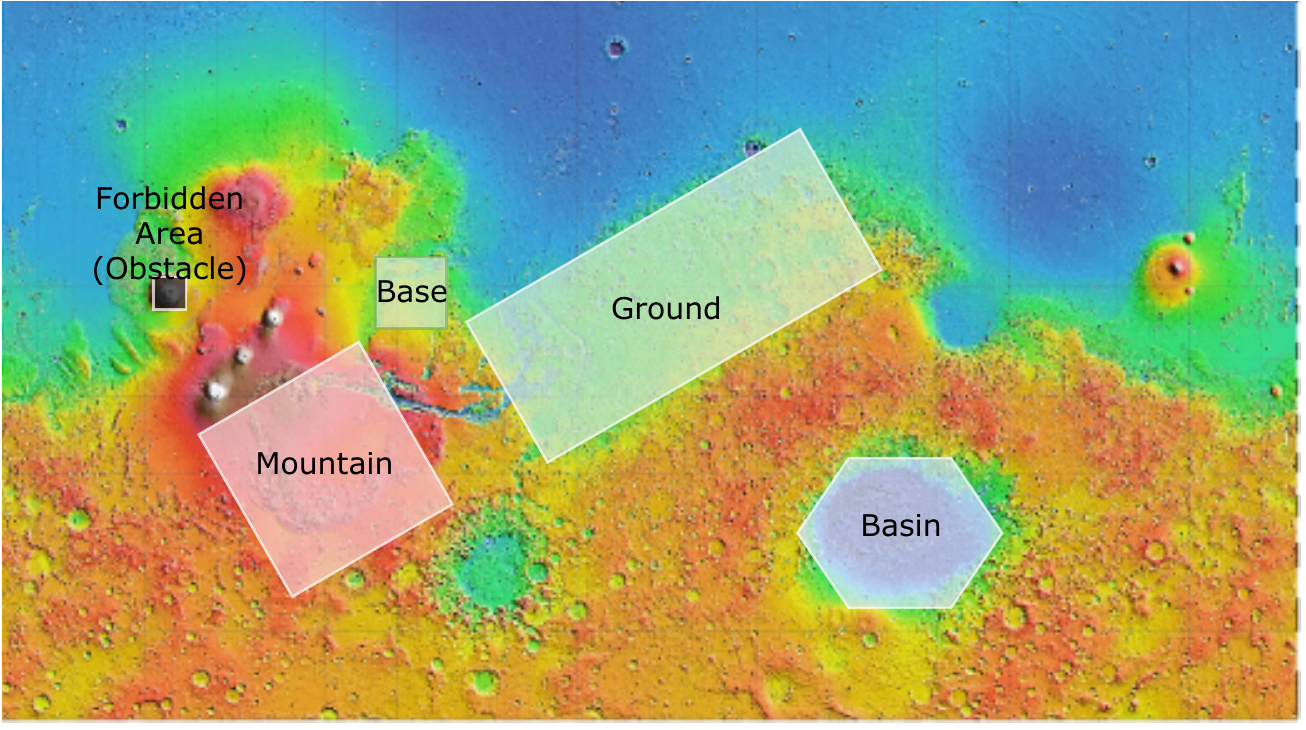}
\caption{\small Map of Mars.}
\label{fig:mars}
\end{figure} 

\begin{figure}[t]
\centering
\includegraphics[width=0.9\columnwidth]{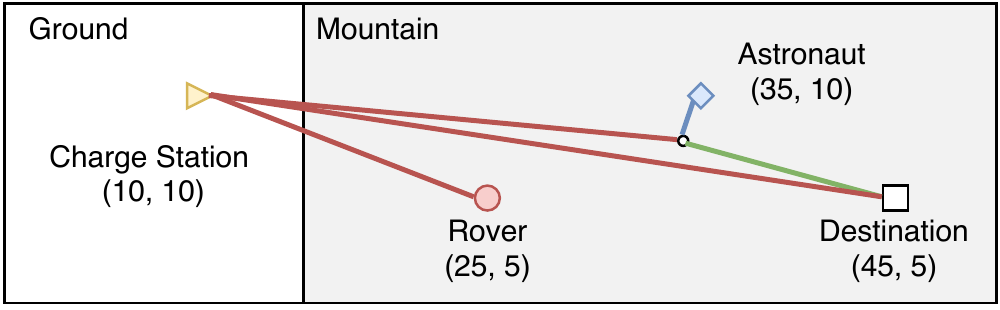}
\caption{\small Example of Mars transportation problems: the charge station is marked as $\triangleright$, and its position is at (10,10); the rover $\circ$ is at (25,5); the astronaut $\diamond$ is at (35,10); and the destination $\Box$ is at (45,5). In the example solution, the rover path is in red, the astronaut walking path is in blue, and the astronaut taking a ride is in green.}
\label{fig:map}
\end{figure} 

\paragraph{Motivating Example} Consider a task where an astronaut should go to an observation location by crossing different terrains (e.g., mountain, ground, and basin) on Mars, and a rover needs to go to the charge station, as shown in Figure~\ref{fig:map}. The astronaut can either walk or take a rover. The moving speed of the rover is much faster than the walking speed of the astronaut. The rover is powered by a battery. This battery can be charged when the rover is stopped in a charge station and should always have the remaining battery outside of the charge station. While the rover and astronaut should not enter the forbidden areas, the rover can move through different terrains with different velocity limits and energy consumption rates. After the rover is manually shut down, it cannot restart within 1 minute. In this mission, plans with shorter makespans are preferred.

A sound solution to this problem is that the astronaut moves directly to the destination, and the rover moves directly to the charge station without picking up the astronaut. However, in this plan, the rover is not used at all. If the rover has enough battery, it can deliver the astronaut to the destination first and then go to the charge station, which saves a lot of time for the astronaut. Unfortunately, energy is always limited in reality. Intuitively, a better solution is that while the astronaut is moving towards the rover, the rover moves to the charge station for charging, then picks up and delivers the astronaut to the destination, and finally returns to the charge station. As the rover moves much faster than the astronaut, this plan requires much less time than letting the astronaut walk to the destination. Another possible faster solution is that the rover picks up and delivers the astronaut somewhere midway to the destination. Then, the astronaut walks to the destination while the rover is moving back to the charge station for a charge.

In this paper, we adopt the hybrid I/O automaton \cite{lynch1995hybrid} framework to model the mixed discrete-continuous planning problems, which is expressive enough to capture all the mentioned features such as discrete and continuous input and state variables, linear dynamics for continuous flows, and guards and resets for discrete transitions. The crucial technique to make the mixed discrete-continuous optimal planning possible is the introduction of finite-step runs. A step is defined either as a discrete jump or as a set of continuous flows in a mode where we need to compute the dwell time for staying in this mode. We prove that the optimal solution (also called a run) with finite steps will converge to the optimal solution of the original hybrid system when the number of steps goes to infinity. By tailoring the automaton to admit only finite-step runs, we show that finding the optimal input such that the corresponding run has a minimum makespan can be encoded as a Mixed Integer Linear Program (MILP), which can be solved efficiently using off-the-shelf MILP solvers, such as Gurobi~\cite{gurobi2020gurobi}.

To encode our mixed discrete-continuous planning problem as a MILP, we draw inspiration from the idea of modeling flow tubes \cite{hofmann2006robust} as linear programs \cite{fernandez2018scottyactivity}. We further extend these linear programs to incorporate discrete decisions to model discrete variables and actions. Our complexity analysis shows that the number of MILP variables and constraints at each step increases linearly with the product of the number of linear constraints in each condition and the number of operators and variables. To accelerate MILP solving, we introduce two types of additional constraints about conflicting operators. We also show that our solution approach is able to plan for the temporally concurrent goals known as Qualitative State Plans (QSPs) \cite{hoffmann2001ff}, which describe the desired system behavior over continuous time as tasks with time windows.

To demonstrate the efficiency and solution qualities of our method, we benchmarked against the Scotty planning system \cite{fernandez2018scottyactivity} on three domains: Mars transportation, air refueling, and truck-and-drone delivery. In addition to dealing with different dynamics under a large number of modes, all these three domains require judiciously coordinating heterogeneous agent teams for cooperation and carefully reasoning over resources to decide necessary recharging or refueling.  The experimental results show that our approach can find high-quality solutions for all the problems in seconds and provide optimality proof for most examples, while Scotty fails to solve half of the problems within 600 seconds. Moreover, the makespans of our first solutions returned within $1$ second are already better than those of Scotty, and our final solutions can significantly improve them.

The remainder of this paper is organized as follows. We start by discussing the related work in both planning and controller synthesis (Section~\ref{section:related}). In Section~\ref{section:formulation}, we give the formal definition of our hybrid planning problem as well as a formulation of our motivating example. In Section~\ref{section:fixed}, we introduce a tractable variant of the hybrid planning problem by fixing the action number of automaton runs, which leads to a finite-step linear hybrid planning problem. In Section~\ref{section:milp}, we present our MILP encoding of this finite-step linear hybrid planning problem. Then, we introduce our extension to deal with temporally concurrent goals in Section~\ref{section:multiple}. Section~\ref{section:results} shows the results of benchmarking our method against the Scotty Planning system on three challenging domains. Finally, concluding remarks are discussed in Section~\ref{section:conclusion}.

\section{Related work}\label{section:related}

Planning~\cite{li2008generative,della2012universal,bryce2016happening,fernandez2018scottyactivity,cashmore2016compilation,bit2019cyber} and controller synthesis~\cite{LTLMOP,paulobook,Girard12,MajumdarMS16,KloetzerB08,Tulip-short,KloetzerB08,WongpiromsarnTM12,TabuadaP06,fan2018controller,raman2015reactive,herbert2017fastrack,vaskov2019towards} methods intersect at finding the (optimal) strategies for various system models with respect to different system specifications. In what follows, we briefly mention a couple of representative approaches that are related, without exhaustively listing all approaches in each category. 

\textit{Discrete abstraction-based synthesis.}
Discrete abstraction-based synthesis methods first compute a discrete, finite-state abstraction of control systems and then synthesize discrete controllers based on automaton theories or two-player games~\cite{paulobook,Girard12,MajumdarMS16,KloetzerB08,Tulip-short,KloetzerB08,WongpiromsarnTM12}. Synthesis tools based on abstraction such as 
CoSyMA~\cite{Cosyma},
Pessoa~\cite{pessoa}, LTLMop~\cite{LTLMOP,LTLMOP2},
Tulip~\cite{Tulip-short,FilippidisDLOM16}, and SCOTS~\cite{rungger2016scots} can support complex nonlinear systems, stochastic uncertainties, or non-deterministic transitions \cite{rungger2016scots,wongpiromsarn2011tulip,filippidis2016control,vidal2001decidable,plaku2013falsification,laurenti2020formal}, and general GR(1)~\cite{wongpiromsarn2011tulip} or Signal Temporal Logic~\cite{raman2015reactive} specifications. Our problem may be solved using abstraction-based synthesis using temporal logic specifications. However, none of the above tools can be used directly on our general linear hybrid system with both discrete input/state signals and guards/resets in transitions. Moreover, our approach aims at finding high-quality solutions with low costs or high rewards over long horizons instead of finding all valid solutions. In fact, our planning approach is efficient and effective at finding high-level plans, which is complementary and can be combined with the controller synthesis algorithms for achieving autonomy in  complex hybrid systems. 


\textit{Sampling-based planning.}
Sampling-based methods such as Probabilistic Roadmaps (PRM)~\cite{kavraki1994probabilistic}, Rapidly-exploring Random Tree (RRT)~\cite{kuffner2000rrt}, Fast Marching Tree (FMT)~\cite{janson2015fast}, and hybrid automata planner~\cite{lahijanian2014sampling} are widely used for searching for plans in nonconvex, high-dimensional, and even partially unknown environments. Researchers also combine the PRM sampling method with classical planning for solving task-and-motion planning problems, which involves both continuous motions and discrete tasks \cite{kaelbling2011hierarchical,lagriffoul2018platform,garrett2015ffrob,garrett2017sample,garrett2018ffrob}. Compared with the deterministic guarantees provided by controller synthesis methods and our approach, these methods come only with probabilistic guarantees.

\textit{Hybrid planning.}
The planning problems with a subset of the features considered in this paper (i.e., mixed discrete-continuous models, continuous control variables, autonomous transitions, indirect effects, and concurrency) can be of the categories PDDL2.1\cite{fox2003pddl2}, PDDL-S \cite{fernandez2018scottyactivity}, or  PDDL+\cite{fox2006modelling}. Some SMT-based PDDL+ planners such as SMTPlan+ \cite{cashmore2016compilation} and dReal \cite{bryce2015smt} are complete given a finite number of fixed time steps. However, PDDL+ does not support control variables and these planners solve different problems from ours. As most of the solution approaches to PDDL2.1 and PDDL-S  \cite{hoffmann2003metric,coles2012colin,fernandez2017mixed} use greedy search methods such as the enforced-hill climbing, they are neither complete nor optimal even with finite steps. Moreover, most of their heuristics belong to the Metric-FF family \cite{hoffmann2003metric}, which are known to suffering from resource persistence or cyclical resource transfer \cite{coles2013hybrid} when indirect effects or obstacles are present. Thus, none of these heuristics can handle all of these problem features and often lead the greedy search to be blind in certain domains. These motivate us to develop an effective method that is guaranteed to provide high-quality solutions for mixed discrete-continuous planning problems with various features. 

Note that hybrid planning problems are closely related to the formalism of hybrid automata studied in model checking \cite{henzinger1998s,lynch1995hybrid}, which can be found in \cite{fox2006modelling}. In addition, researchers put many efforts into translating PDDL+ to hybrid automata \cite{bogomolov2014planning,bogomolov2015pddl+}, which leverages the advanced hybrid model checking tools \cite{cimatti2000nusmv,frehse2011spaceex,frehse2008phaver} to efficiently prove plan non-existence. Our method directly plans for hybrid automata instead of any PDDL extension. These two representations can be translated to each other since snap actions are basically jumps, and the overall conditions and effects of durative actions are basically flows. By using jumps and flows instead of durative actions, we are able to have a clean MILP encoding for hybrid planning problems.

Among all the hybrid extensions of PDDL, the problems this paper aims at are most relevant to PDDL-S since it is the only planning formalism that supports continuous control variables over time \cite{fernandez2018scottyactivity}.  Kongming \cite{li2008generative} is the first planner that is able to solve PDDL-S, and then a more scalable planner Scotty \cite{fernandez2018scottyactivity} was developed. Scotty is able to efficiently solve complex underwater exploration problems, and it is the current state-of-the-art PDDL-S planner. The reasons for its efficiency are: (1) Scotty encodes the cumulative effect of each control variable as a single variable, which renders a clean convex optimization problem for plan validation, which are called flow tubes \cite{hofmann2006robust}; (2) It uses the temporal relaxed planning graph heuristics (i.e., delete relaxation) \cite{coles2012colin} to guide its greedy search.  Our method is inspired by its cumulative effect encoding and extends its optimization problem to handle discrete decisions and nonconvex conditions. By using such an encoding, we do not need to discretize the timeline with a fixed time step as \cite{li2008generative} or discretize control parameters as \cite{coles2012colin}. Meanwhile, our solution approach avoids the incompleteness and suboptimality caused by Scotty's greedy search.

\section{Problem Formulation}\label{section:formulation}
We use a linear hybrid automaton with inputs as the model for our system and then define the hybrid planning problem on this automaton.
\begin{definition}[Linear Hybrid Automaton] 
\label{def:HA}
A linear hybrid automaton with inputs is a tuple $\hybrid = \langle V = (Q \cup E), \init, \goal, J, F \rangle $:
\begin{itemize}
    \item $Q = L \cup X$ is the set of {\em internal} variables, which are also called {\em state} variables. $L$ is the set of discrete state variables, the values of which $\val(L)$ are taken from finite sets called modes. $X$ is the set of continuous state variables, the values of which $\val(X)$ are taken from continuous sets over the reals. We call $\val(L) = \val(L) \times \val(X)$ internal state space.
    
    \item  $E$ is the set of {\em external} variables, which are also called {\em input} variables. External variables could also contain discrete and continuous variables, which are defined analogously to the internal variables. 
    
    \item $\init \in \val(L) \times \val(X)$ is an initial state, and $\goal$ is a predicate that represents a set of goal states.

    \item $J$ is the set of {\em jumps}. A jump $j \in J$ is associated with a {\em condition} $\textit{cond}$ and an {\em effect} $\textit{eff}$. The condition $\textit{cond}$ is a {\em predicate} over $V$, where a predicate is a computable Boolean-valued function $\textit{cond} : \val(V) \rightarrow \bool$ that maps the values of the variables $V$ to either $\true$ or $\false$. The condition is also known as the guard condition or the enabling condition of the jump.
    An effect $\textit{eff}: V \rightarrow Q$ specifies how the value of the state variables changes when the jump occurs. It assigns new values to the variables in $Q$.
    The variables that are not mentioned in the effect statements are assumed to remain unchanged.
    
    \item $F$ is the set of {\em flows} for the state variables $X$. $F_k \subseteq F$ is the set of flows for $X_k \subseteq X$, where $\{X_0,X_1,..,X_K\}$ is a set of disjoint continuous variable sets such that $\cup_{k} X_k = X$ . A flow $f \in F_k$ is associated with a differential equation $\dot X_k = A_k E + B_k$ and a condition $\textit{cond}$ over $V$, where $A_k, B_k$ are constant matrices, and $\textit{cond}$ is defined in the same way as in jumps that specifies when a flow $f$ is activated.
    At each time, multiple flows $f \in F$ can be activated with exactly one flow $f_k$ from each $F_k$. That is, there will always be a set of flows, which together specify the evolution of the continuous internal variables $X$ as linear differential equations. We call such a set of flows the {\em flow set} at each time, and it belongs to the power set of $F$. During the time when a flow is activated, the values of discrete state variables stay the same.
    
\end{itemize}
\end{definition}

    
    
    
    

Note that $\val(Q)$ also defines the invariant set of the internal variables, where $\val(X)$ could be nonconvex. Therefore, we can avoid defining the unsafe set separately.

Without loss of generality, we use an integer variable with domain $\{0,1,..,|\val(v)|-1\}$ to replace a discrete variable $v \in V$, where $|\val(v)|$ is the number of the elements in $\val(v)$. Thus, we can further assume that all conditions, initial states, and goals are represented as a propositional sentence of liner constraints:
\begin{equation}\label{eq:formula}
    \phi ::= \true \mid (GV \geq H) \mid 
  \phi_1 \mid \phi_1 \wedge
  \phi_2 \mid \phi_1 \vee
  \phi_2,
\end{equation} 
where $G \in \reals^{|V|}$ is a $|V|$-vector of real values and $H \in \reals$ is a real value. This propositional sentence of linear constraints can represent both convex and nonconvex regions defined by linear inequalities over both integer and continuous variables. Effects can be also represented by \eqref{eq:formula} except its linear constraints involve both $V = Q \cup E$, which are the state and input variables before taking the effects, and $Q'$, which is the state variables after taking the effects.

\begin{example} 
To formulate our motivating example, we define two discrete internal variables: the astronaut $\texttt{LA} \in \{\texttt{0},\texttt{1}\}$ has modes $\texttt{Walking(0)}$ and mode $\texttt{Riding(1)}$; the rover $\texttt{LR} \in \{\texttt{0},\texttt{1},\texttt{2}\}$ has $\texttt{Driving(0)}$, $\texttt{Stopped(1)}$ and $\texttt{Charge(2)}$.

Internal continuous variables $\texttt{pA} \in [0,50]\times[0,30]$ represent the astronaut's position, and $\texttt{xR} \in [0,50]\times[0,30]\times[0,30] \times [0,\infty)$ includes the rover's position $\texttt{pR} \in [0,50]\times[0,30]$, battery level $E \in [0,30]$, and an internal clock $c \in [0, \infty)$. $\texttt{pRx}$ and $\texttt{pRy}$ are the rover's positions over the x-axis and the y-axis, respectively.  The initial state is $\texttt{LA}=\texttt{0},\texttt{LR}=\texttt{1},\texttt{pA}=(35,10),\texttt{xR}=(25,5,10,0)$ as shown in Figure~\ref{fig:map}. 

This system also takes commands as input variables, including discrete input variables $\texttt{cmdA} \in \{0,1\}$,  $\texttt{cmdR} \in \{0,1,2\}$, and continuous input variables $\texttt{vA} \in [-0.2,0.2]^2$, $\texttt{vR} \in [-5,5]^2$ represent velocities.

\begin{lstlisting}[]
    |\textbf{Jumps}| |\vspace{1pt}|
    Board   |\textbf{cond}| (LA=0)$\land$(cmdA=1)$\land$(pA=pR)$\land$(vR=0) |\textbf{eff}| (LA=1) |\vspace{1pt}|
    Deboard |\textbf{cond}| (LA=1)$\land$(cmdA=0)$\land$(vR=0) |\textbf{eff}| |\vspace{1pt}| (LA=0)
    Stop    |\textbf{cond}| (cmdR=0) |\textbf{eff}| (LR=1),(c=0) |\vspace{1pt}|
    Drive   |\textbf{cond}| (LR=1)$\land$(cmdR=1)$\land$(c>1) |\textbf{eff}| (LR=0)| \vspace{1pt}|
    Charge  |\textbf{cond}| (LR=1)$\land$(cmdR=2)$\land$(pR=(10,10)) |\textbf{eff}| (LR=2) |\vspace{1pt}|
    |\textbf{Flows}| |\vspace{1pt}|
    Ride   $\textbf{eq}$ $\dot{\texttt{pA}}(t)$=vR |\textbf{cond}|(LA=1)$\land$(cmdA=1)$\land$(pA=xA) |\vspace{1pt}|
    Walk   $\textbf{eq}$ $\dot{\texttt{pA}}(t)$=vA |\textbf{cond}| (LA=0)$\land$(cmdA=0) |\vspace{1pt}|
    Ground $\textbf{eq}$ $\dot{\texttt{xR}}(t)$= [vR,-1,0] |\textbf{cond}| (LR=0)$\land$(cmdR=1)$\land$(pRx<20) |\vspace{1pt}|
    Mount  $\textbf{eq}$ $\dot{\texttt{xR}}(t)$= [vR,-2,0] |\vspace{1pt}| |\textbf{cond}| (LR=0)$\land$(cmdR=1)$\land$(pRx>20)$\land$($|\texttt{vR}|$<2) |\vspace{1pt}|
    Stop   $\textbf{eq}$ $\dot{\texttt{xR}}(t)$= [0,0,0,1] |\textbf{cond}| (LR=1)$\land$(cmdR=0) |\vspace{1pt}|
    Charge $\textbf{eq}$ $\dot{\texttt{xR}}(t)$= [0,0,0,10] |\textbf{cond}| (LR=2)$\land$(cmdR=2) |\vspace{1pt}|
\end{lstlisting}
\end{example}

While flow sets can only change continuous state variables, jumps can change both discrete and continuous state variables. The conditions for both jumps and flows would depend on all variables (i.e., including state and input variables). While we call the union of jumps and flows $J \cup F$ as \textit{operators} $O$, we call both jumps and flow sets $ J \cup 2^{F}$ as \textit{actions} $A$. We use $\texttt{cond}_a$ to denote the set of states $v \in \val(V)$ such that the condition associated with the action $a$ is true: $\phi(v) = True$. 


Given a flow set $a  = \cup_k f_k \in 2^{F}$, we denote the derivative of $X$ as $AE+B$. Such $A$ and $B$ can be easily constructed from the differential equations for each flow $X_k = A_k E + B_k$. 
If at the beginning of the flow the value of $X$ is $x_0$, and the elapsed time of such flow is $\delta$, then the $X$'s value would be updated as 
$x \gets x_0 + A \Delta + B\delta$, where $\Delta = \int_{0}^{\delta} E dt$ is the cumulative effects of $E$ during $d$. 

An input signal is a function $e: [0, \infty) \rightarrow \val(E)$, which specifies the value of the input variables at any time $t \geq 0$.
Once an input signal is fixed, a run of the hybrid automaton is defined as follows:

\begin{definition}
\label{def:run}
Given a linear hybrid automaton $\hybrid = \langle V = (Q \cup E), \init, \goal, J, F \rangle $ and an input command $e: [0, \infty) \rightarrow \val(E)$, a {\em run}~\footnote{We assume that in our run Zeno behaviors are not allowed. That is, we do not allow an infinite number of jumps to occur in a finite time interval.} of $\hybrid$ is defined as a sequence of internal states $q_0,\cdots,q_n \in \val(L) \times \val(X)$:
\[
\xi_{\hybrid,e} = q_0 \xrightarrow{a_0, \delta_0} q_1 \cdots,q_{n-1} \xrightarrow{a_{n-1}, \delta_{n-1}} q_n,
\]
such that
\begin{enumerate}
    \item $q_0 \in \init$ and $q_n \in \goal$.
    \item $a_0,\cdots,a_{n-1}$ are actions. Let $t_i = \sum_{j=0}^{i-1}\delta_j$ be the accumulated time associated with $q_i$ for each $i = 0,\cdots,n$, then: (a) if $a_i \in J$ is a jump, then $\delta_i = 0$, $(q_i, e(t_i)) \in \mbox{cond}(a_i)$, and $q_{i+1} = \mbox{eff}(q_i, e(t_i))$; (b) if $a_i \in 2^F$ is a flow set, then $\delta_i \geq 0$, $(q_i, e(t_i)) \in \mbox{cond}(a_i)$, $x_{i+1} = x_i + A \int_{t_i}^{t_{i+1}} e(\tau) d\tau + B \delta_i$, $\ell_{i+1} = \ell_{i}$ where $q_i = (\ell_i, x_i)$. Moreover, between the time $t \in [t_i, t_{i+1})$, $(q(t),e(t))$ should always satisfy cond($a_i$). 
\end{enumerate}
We also denote the total time $\sum_{i=0}^{n-1} \delta_i$ of a run $e$ as $\xi_{\hybrid,e}.\mathtt{TotalTime}$. Note that although $e$ is defined on the infinite time horizon $[0,\infty)$, we do not need to have the value for $e(t)$ when $t > \xi_{\hybrid,e}.\mathtt{TotalTime}$
\end{definition}

Now given a linear hybrid automaton with inputs, we can define the planning problem as finding an input signal whose run has the minimum makespan.
\begin{definition}\label{def:solution}
Given a linear hybrid automaton $\hybrid = \langle V = (Q \cup E), \init, \goal, J, F \rangle $, the planning problem is to find a the optimal input $e^*$ signal so the corresponding solution's makespan is minimized:
\[
e* = \argmin_{e} \xi_{\hybrid,e}.\mathtt{TotalTime}
\]
\end{definition}

\section{finite-step Decision Problem}\label{section:fixed}

Solving the planning problem (Definition~\ref{def:solution}) to get the optimal input signal $e^*$ needs to reason over all possible $e(t)$. For most hybrid automaton, this is intractable, as the unbounded-time reachability problem is undecidable even for rectangular hybrid automaton~\cite{henzinger1998s}, which is a simpler hybrid automaton than ours with the right-hand side of the differential equations containing only constants.

Essentially, to solve the optimal $e^*$, we need to assign values of all input variables for infinitely many $t$. To make this problem solvable, we fix the number of actions allowed in the run of the hybrid automaton and simplify the original problem by searching for $e(t)$ that corresponds to each action. We introduce a {\em fixed-step linear hybrid automaton} to capture such an idea.

\begin{definition}
A {\em finite-step linear hybrid automaton with input} $\hybrid_n$ is a linear hybrid automaton $\hybrid$ (as defined in Definition~\ref{def:HA}) with all runs of $\hybrid$ to have exactly $n$ actions.
\end{definition}


In Section~\ref{section:milp}, we present how to use a MILP encoding to solve the planning problem for fixed-step linear hybrid automaton. Next, we show that for a hybrid automaton $\hybrid$, once we fix the number of actions $n$ and make it $\hybrid_n$, the feasible solution set (the set of input signals $e$ such that $\xi_{\hybrid_n, e}$ is a run of $\hybrid$ with $n$ actions) is non-decreasing as the number of actions $n$ increases.

\begin{lemma}\label{lemma:step}
Let $\hybrid^n$ be a finite-step linear hybrid automaton for $\hybrid$ with fixed action number $n$. Let $\Xi_{\hybrid^n}$ and $\mathcal{E}_{\hybrid^n}$ be all the runs of $\hybrid^n$ and their corresponding input signals, respectively. For any $0 < n' < n$, we have $\mathcal{E}_{\hybrid^{n'}} \subseteq \mathcal{E}_{\hybrid^n}$.
\end{lemma}

\begin{proof}
For any $e \in \mathcal{E}_{\hybrid^n}$, let $ q \xrightarrow{a, \delta} q'$ be a segment of its run $\xi_{\hybrid, e}$ and $a \in 2^F$ is a flow set. We can replace this segment with $q \xrightarrow{a, \delta} q' \xrightarrow{a, 0} q'$, and the new run is still a run of $\hybrid$ given input signal $e$. As this new run has $n+1$ actions, we prove that $e \in \mathcal{E}_{\hybrid^{n+1}}$ for any $e \in \mathcal{E}_{\hybrid^{n}}$.
\end{proof}

As the original hybrid automaton $\hybrid$ does not fix the action number, we know $\mathcal{E}_{\hybrid}  = \bigvee_{n = 0}^{n=\infty} \mathcal{E}_{\hybrid^n} = \lim_{n \rightarrow \infty}\mathcal{E}_{\hybrid^n}$, which directly follows from Lemma~\ref{lemma:step}.

Let 
\begin{equation}
\label{eq:enstar}
 e_n^* = \argmin\limits_{e \in \mathcal{E}_{\hybrid^{n}}} \xi_{\hybrid_n,e}.\mathtt{TotalTime}.   
\end{equation}
As $\mathcal{E}_{\hybrid^{n'}} \subseteq \mathcal{E}_{\hybrid^n}$, it is easy to check that $\xi_{\hybrid_{n'},e_{n'}^*}.\mathtt{TotalTime} \geq \xi_{\hybrid_n,e_n^*}.\mathtt{TotalTime}$. This gives us the following corollary.

\begin{corollary}
Following Lemma~\ref{lemma:step}, then $\lim\limits_{n \rightarrow \infty} e_n^* = e^*$, where $e_n^*$ is defined as \eqref{eq:enstar}.
\end{corollary}


\section{Mixed Integer Linear Encoding}
\label{section:milp}
In this section, we describe how to encode a finite-step linear hybrid planning problem as a Mixed Integer Linear Program, in which numbers of variables or constraints at each step increase linearly with the product of the operator number and the number of disjuncts involved in each condition (Section~\ref{sec:milp:complexity}). We first introduce a method to encode formulas with syntax as in \eqref{eq:formula} (Section~\ref{section:milp:disjunctive}), and move on to the detailed encoding procedure for the finite-step linear hybrid problem (Section~\ref{sec:milp:constraints}). Additional constraints about conflicting operators for speeding up MILP solving are discussed in (Section~\ref{sec:milp:additional}).

\subsection{Encoding Linear Constraint Formulas}\label{section:milp:disjunctive}
Firstly, we introduce the methods to encode constraints with syntax \eqref{eq:formula} as MILP constraints. We start by encoding a canonicalized form and move on to the general case. 

\paragraph{Encoding CNF Linear Constraint Formula} Note that a condition expressed using \eqref{eq:formula} can be always transformed into a conjunctive normal form (CNF) of linear constraints:
\begin{equation}
    \texttt{cond}(V) \equiv \land_r^m \lor_s^{m_r} (\texttt{cond}_{rs}(V)) \equiv \land_r^m \lor_s^{m_r} (G_{rs} V \geq H_{rs}),    
\end{equation}
where $G_{rs} \in \reals^{|V|}$ and $H_{rs} \in \reals$, $m$ is the number of conjuncts in $\texttt{cond}$, and $m_r$ is the number of disjuncts in the $r^{\text{th}}$ conjunct. For convenience, we also replace $G_{rs}V > H_{rs}$ with $G_{rs}V \geq H_{rs}$ without invalidating the solutions. As disjunctions are present in $\texttt{cond}$, which result in nonconvex sets in general, we use the Big-M method to handle such disjunctive logic in  $\lor_s^{m_r} (G_{rs} V \geq H_{rs})$. We define a $m_r$-vector of intermediate Boolean variables $\alpha_r$ with domain $\{0,1\}$. While $G_{rs} V \geq H_{rs}$ should hold if $\alpha_{rs} = 1$, we have $\alpha_{rs} = 1$ for at least one $\alpha_{rs}$. Then, the $r^\text{th}$ disjunction $\lor_s^{m_r} (G_{rs} V \geq H_{rs})$ is represented as a set of linear constraints:
\begin{equation}
  \left( \land_s^{m_r} \left(  \alpha_{rs} = 1  \implies G_{rs} V \geq H_{rs} \right) \right) \land  \left( \sum_s^{m_r} \alpha_{rs} \geq 1 \right)   
\end{equation}

Let $M$ be a very large positive number, then each implication is encoded as linear inequalities over both $V$ and indicator variables:
\begin{equation}
G_{rs}V + M(1-\alpha_{rs}) \geq H_{rs}
\end{equation}

As $\texttt{cond}(V) = \true$ needs all the conjuncts to hold, we end up with the following constraint:
\begin{equation}\label{eq:cnf}
    \land_r^m \land_s^{m_r}  \left( \left( G_{rs}V + M \left( 1-\alpha_{rs} \right) \geq H_{rs} \right) \land  \left( \sum_s^{m_r} \alpha_{rs} \geq 1 \right) \right).
\end{equation}

\paragraph{Encoding General Linear Constraint Formula} 
While some regions are easier to encode by using CNFs, the CNFs of some others take more space. For example, in Figure~\ref{fig:segment}, the blue region is a good candidate to be encoded as CNF $(\land_0^3 \phi_{0i}) \land (\lor_0^3 \phi_{1i}) \land (\lor_0^3 \phi_{3i})$, where $\phi_{ij}$ is a linear inequality and its direction is given in the figure. However, the green region is more intuitive to  be encoded as $ (\land_0^3 \phi_{2i}) \lor (\land_0^3 \phi_{4i})$, which is not CNF.

\begin{figure}[ht]
\centering
\includegraphics[width=0.9\columnwidth]{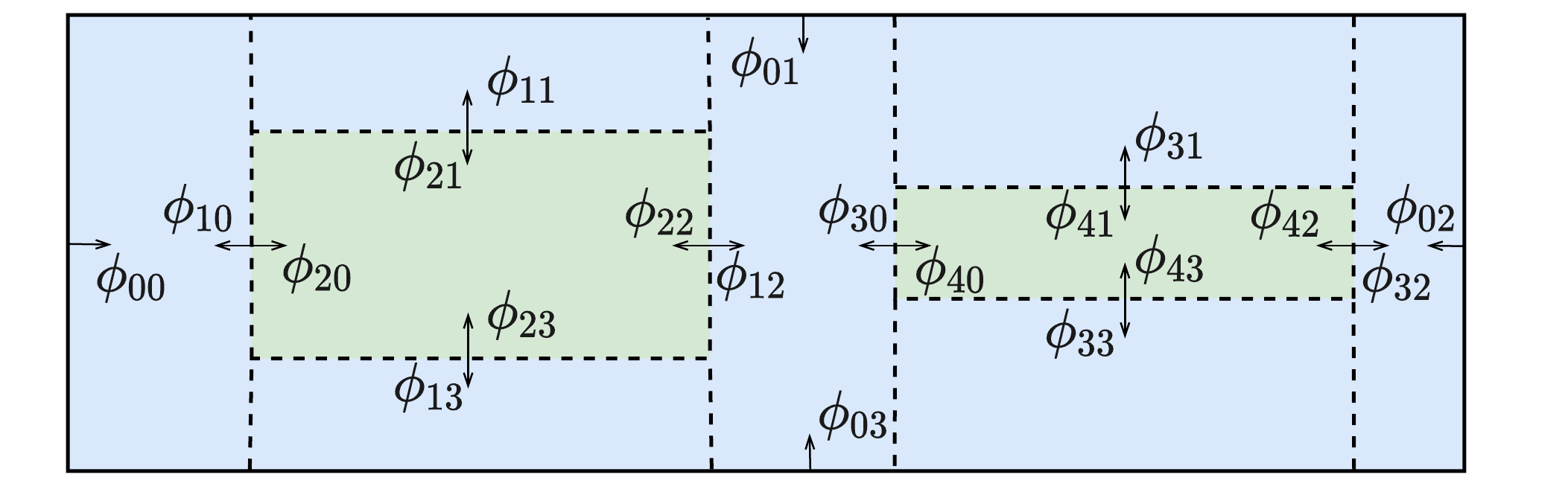}
\caption{\small Examples of two regions to encode as linear constraint formulas.}
\label{fig:segment}
\end{figure} 

We can use similar methods to encode the general case in \eqref{eq:formula} in a recursive fashion. Given a set of formulas $\{\phi_0, \phi_1,..\phi_m\}$, we assume the MILP constraint for $\phi_r$ is already encoded as $\land_s^{m_r} (\texttt{LHS}_{rs} \geq \texttt{RHS}_{rs})$, which is a set of linear constraints over $V$ and some indicator variables. We show both the conjunction and disjunction encodings of these formulas. To obtain the MILP constraint of conjunction $\land_r^m \phi_r$, we have:
\begin{equation}\label{eq:and}
    \land_r^m \land_s^{m_r} (\texttt{LHS}_{rs} \geq \texttt{RHS}_{rs}).
\end{equation}

To encode disjunction $\lor_r^m \phi_r$ , we introduce an $m$-vector of Boolean variables $\alpha$ and the disjunction is captured by: 
\begin{equation}\label{eq:or}
     \left( \land_r^m \land_s^{m_r}  \left( \texttt{LHS}_{rs} + M \left( 1-\alpha_{r} \right) \geq \texttt{RHS}_{rs} \right)  \right)   \land  \left( \sum_r \alpha_{r}^m \geq 1 \right) .
\end{equation}
For each linear constraint, if there exist some indicator variable that is not equal to $1$, the constraint trivially holds. This set of constraints can be further canonicalized into $\land_s^{m_r} (\texttt{LHS}_{rs} \geq \texttt{RHS}_{rs})$ over $V$ and indicator variables for other logic operations. 


\subsection{Encoding Finite-step Hybrid Planning}
Now we are ready to encode the entire planning problem as defined in Definition~\ref{def:solution} for linear hybrid automata with $n$-step runs.
\label{sec:milp:constraints}
\subsubsection{Variables}\label{section:milp:var}

To represent the internal states $\{q_0, q_1, .., q_n\}$, we define a set of variables $\{Q_0, Q_1,.., Q_n\}$, and $Q_i$ corresponds to the internal state $Q_i$ right after $a_i$ occurs and right before $a_{i+1}$ occurs. Their domains are copied from $Q$. To represent the input signal $e$, we also have $E_i \in \{E_0,E_1,..,E_{n-1}\}$, corresponding to the values of $E$ when $a_i$ occurs, and their domains are copied from $E$.  

To represent the actions that happen at each step, we define a set of binary activation variables $\{P_0, P_1,..,P_{n-1}\}$. $P_i$ is the union of $P^J_i$ and $P^F_i = P^{F_0}_i \cup P^{F_1}_i \cup,..,P^{F_K}_i$, which are the activation variables at step $i$ for jumps $J$ and flows $\{F_1, F_2,..,F_K\}$, respectively. Each $p_i^o \in P_i$ corresponds to an operator $o$ (i.e., a jump or a flow) at step $i$. If $p_i^o = 1$, operator $o$ is activated at step $i$; otherwise, $o$ is inactivated.  To fully determine the effects of flows, we need to specify the cumulative effects of the input variables and the elapsed time during these flows. Thus, we define $d_i$ with domain $[0,\infty)$ to represent the elapsed time during step $i$; and real variable $\Delta$ denotes $\int_{0}^{d_i} E_i dt$, the cumulative effects of $E_i$ during step $i$.

We denote the set of all these MILP variables as $\mathcal{V}$. Let $\Pi: \mathcal{V} \rightarrow \reals$ be a MILP solution that maps a MILP variable $v \in \mathcal{V}$ to a value $\Pi(v) \in \val(v)$. Then, given a MILP solution $\Pi$, we can get the values of the input command $e$ as well as a valid run $\xi_{\hybrid,e} = q_0 \xrightarrow{a_0, \delta_0} q_1 \cdots,q_{n-1} \xrightarrow{a_{n-1}, \delta_{n-1}} q_n$ as given in Definition~\ref{def:run}. 
We extract $e$ over duration $[0, \sum_{i=0}^{(n-1)} \Pi(d_i)]$ as follows:
\begin{equation}\label{eq:extract_e}
    \begin{aligned}
    e(t)=
    \left\{
    \begin{array}{l}
    
    \Pi(E_i), \quad \text{if } t =  (\sum_{j=1}^i \Pi(d_j)) \\
    
    \Pi(\Delta_i)/\Pi(d_i), \quad \text{if } (\sum_{j=0}^{(i-1)} \Pi(d_j)) < t < (\sum_{j=0}^{i} \Pi(d_j))
    
    \end{array}
    \right.
    \end{aligned}
\end{equation}

We extract the run $\xi_{\hybrid,e}$ of $e$ from $\Pi$ as follows:
\begin{equation}\label{eq:extract_run}
    \begin{aligned}
    q_i = \Pi(Q_i), & \quad \text{for } i \in \{0,1,..,n\}, \\
    \delta_i = \Pi(d_i),& \quad \text{for } i \in \{0,1,..,n-1\}, \\
    a_i = a \text{ if } p \in P_i \text{ and } \Pi(p) = 1, & \quad \text{for } i \in \{0,1,..,n-1\}.
    \end{aligned}
\end{equation}


\subsubsection{Objective Function}

As specified in Definition~\ref{def:run}, we aim at finding a run with minimum $\texttt{TotalTime}$, and thus the objective function is $\sum_{i=0}^{n-1} \delta_i$.

\subsubsection{Constraints}
Next, we introduce the constraints over these variables, which force only one jump or one flow set to be chosen at every time step with their conditions being satisfied and effects being imposed, such that the goal states can be reached from the initial state through these actions. Figure~\ref{fig:conditions} summarizes all constraints C1-C10 to encode the planning problem. We have the following theorem that justify the pair of input command and run given by \eqref{eq:extract_e}-\eqref{eq:extract_run} is valid:

\begin{theorem} \label{theorem:justify}
Given a $n$-step hybrid automata $\hybrid$ and a MILP solution $\Pi$ over the variables $\mathcal{V}$ as defined in Section~\ref{section:milp:var}, the input command $e$ and run $\xi_{\hybrid,e}$ that are extracted from $\Pi$ by  \eqref{eq:extract_e}-\eqref{eq:extract_run} are an optimal solution of $\hybrid$ if $\mathcal{V}$ satisfies constraint C1-C10 in Figure~\ref{fig:conditions} and $\sum_{i=0}^{n-1} \delta_i$ is minimized.
\end{theorem}

Next, we prove Theorem~\ref{theorem:justify} by explaining C1-C10 in detail. As these constraints are the exact translation of Definition~\ref{def:run}, our solution approach is sound and complete and thus optimal.

\begin{figure}[t]
\centering
\begin{lstlisting}[basicstyle=\small]
|\textbf{Initial and goal states}| |\vspace{5pt}|
C1. $(Q_0= q_{\mathtt{Init}}) \land (\goal(Q_n) = \true)$ |\vspace{5pt}|
|\textbf{Operator activation}| |\vspace{5pt}|
C2. $\bigwedge_{i = 0}^{n-1} \bigwedge_{F_k \subseteq F} (\sum_{p_i^o \in (P^J_i \cup  P^{F_k}_i)} p_i^o = 1)$ |\vspace{5pt}|
|\textbf{Jump constraint}| |\vspace{5pt}|
C3. $\bigwedge_{i = 0}^{n-1} \bigwedge_{j \subseteq J}  ((p^j_i = 1) \implies \texttt{cond}_j(V_i))$ |\vspace{5pt}|
C4. $\bigwedge_{i = 1}^{n} \bigwedge_{j \subseteq J}  ((p^j_i = 1) \implies (Q_i = \texttt{eff}_j(V_{i-1})))$ |\vspace{5pt}|
C5. $\bigwedge_{i = 0}^{n-1} \bigwedge_{j \in J}  ((p^j_i = 1) \implies (d_i = 0))$|\vspace{5pt}|
|\textbf{Flow constraint}| |\vspace{5pt}|
C6. $\bigwedge_{i = 0}^{n-1} \bigwedge_{f \subseteq F} (p_i^f = 1) \implies (\bigwedge_r \bigvee_s (\textit{cond}^Q_{f,rs} (Q_{i}) \land \textit{cond}^Q_{f,rs} (Q_{i+1}))) $|\vspace{5pt}|
C7. $\bigwedge_{i = 0}^{n-1} \bigwedge_{f \subseteq F} (p_i^f = 1) \implies \textit{cond}^E_{f} (E_{i})  $ |\vspace{5pt}|
C8. $\bigwedge_{i = 0}^{n-1} \bigwedge_{f \subseteq F} (p_i^f = 1) \implies \textit{cond}^\Delta_{f} (\Delta_{i}, d_i)  $ |\vspace{5pt}|
C9. $\bigwedge_{i = 0}^{n-1} \bigwedge_{f \subseteq F} ((p^f_i=1) \implies (X_{k(i+1)} - X_{ki} = A_f \Delta_i - B_fd_i))$ |\vspace{5pt}|
C10. $\bigwedge_{i = 0}^{n-1} \bigwedge_{f \in F}  ((p_i^F = 1) \implies (L_{i+1} = L_i))$|\vspace{5pt}|

\end{lstlisting}
  \caption{\small Constraints in the MILP encoding.}
  \label{fig:conditions}
\end{figure}

\paragraph{Initial and Goal States}
First, constraint C1 ensures that runs start from $\init$, and its final state $Q_n$ satisfies $\goal$ such that Definition~\ref{def:run}(1) is respected. We use \eqref{eq:and}-\eqref{eq:or} to encode constraint $\goal(Q_N) = \true$.

\paragraph{Operator Activation}
Constraint C2 can avoid the ambiguity of having multiple jumps or multiple flows for the same internal continuous variables at each step. Recall that $p_i^o = 1$ means the corresponding operator $o$ is activated at step $i$. Constraint C2 forces either of the following conditions to hold: (1) exactly one jump is active, and all the flows are inactivated (Definition~\ref{def:run}(2a)); (2) all jumps are inactivated, and there is exactly one flow for each continuous variable set is activated (Definition~\ref{def:run}(2b)).

\paragraph{Jump Constraint}
When a jump is active, its conditions should be satisfied, and their effects should be imposed (Definition~\ref{def:run}(2a)). For each jump $j \in J$ with condition $\textit{cond}_j$, when jump $j$ is activate at step $i$, which is $p_{i}^j = 1$, condition $\textit{cond}_j$ should hold, which is captured by C3. Constraint C4 enforces the effect, which is linear constraints $Q_i = \texttt{eff}_j(V_{i-1})$, to happen right after jumps. Note that this constraint also forces the unaffected variables to remain the same after the jumps. In addition, C5 ensures that the elapsed time during jumps is zero, which is activated when some jump is chosen.

\paragraph{Flow Constraint}
While the condition of jumps should only hold right before it happens, the condition of flows should always hold until the next action (Definition~\ref{def:run}(2b)).
Let $\textit{cond}_f$ be the condition of a flow $f \in F$ and denote the constraints of $\textit{cond}_f$ over $Q$ and $E$ as $\textit{cond}_f^Q$ and $\textit{cond}_f^E$, respectively. Given our solution specification, while the constraint over $Q$ should hold through the execution of $f$, including the start and end, the constraint over $E$ should also thoroughly hold except at the end of this flow, where the change of $E$ may trigger other jumps or flows. 

Since the considered dynamics are linear and all the conditions are sets of disjunctive linear constraints, satisfying $\textit{cond}_f ^ Q$ at the start and the end of flow $f$ with respect to the same disjunct in each disjunctive linear constraint implies $\textit{cond}_f ^ Q$ holds during $f$. Constraint C6 captures $\textit{cond}_f ^ Q$ at the start and end of flow $j$ being activated. This constraint is sufficient to guarantee the linear trajectory from $Q_i$ ti $Q_{i+1}$ always satisfies $\texttt{cond}^Q_j$, and the reason is as follows: they share the same indicator variables and always satisfy the same disjunct in each disjunction, and thus they are in the same convex region; as the line between two points in a convex region always stay in this region, we know that the trajectory from $Q_i$ to $Q_{i+1}$ satisfy $\texttt{cond}_f^Q$. As a similar encoding for motion planning with polytope obstacles can be seen in \cite{FanMMV:CAV2018}, our encoding method extends it to handle more general linear constraint formula.

As condition $\textit{cond}_f ^ E$ should hold through flow $f$ except at the end, we add constraints over $E$ when $f$ starts and the constraints over $\Delta$, which is the cumulative effects of $E$ during $f$ happening. While the former constraint is captured in C7, the latter is in C8. For a linear constraint $G^E_{f, rs}E_i \geq H^E_{f, rs}$, we can obtain the equivalent linear constraint over $\Delta_i$ and $d_i$ by integrating it over time on both sides and substituting in $\Delta = \int_{0}^{d} E dt$:
\begin{equation}
    G^E_{f, rs} \Delta_i \leq H^E_{f, rs} d_i.
\end{equation}
By doing this for each  $G^E_{f, rs}E_i \geq H^E_{f, rs}$ in $\texttt{cond}^E_{f, rs}$, we obtain the condition $\texttt{cond}^\Delta_{f, rs}$ over $(\Delta, d)$.

Then, we determine the evolution of state variables during the flows. Recall that the state variables $Q$ consist of continuous state variables $X$ and discrete state variables $L$.  Let the differential equation of a flow $f$ be $\dot X_k = A_fE + B_f$, and then C9 enforces continuous dynamics by adding the effects of $\Delta$ and $d$ to $X_k$. Constraint C10 makes sure that flows do not change discrete variables.

\subsection{Complexity Analysis}
\label{sec:milp:complexity}

Now, we discuss the complexity of our MILP encoding. Let $n$ be the number of total steps, $K$ be the total number of disjoint continuous variable sets $\{X_1,X_2,..,X_L\}$, $Q$ be the internal state variables, $E$ be the input variables, $J$ be the jumps, $F$ be the flows, and $m$ and $m'$ be the maximum number of linear inequalities and disjuncts in each condition or effect, respectively. As shown in Section~\ref{section:milp:var}, we know the total number of variables in a MILP is the sum of the following variables: the variables for internal states, input signals, elapsed times, cumulative effects, which are $n(|Q|+2|E|+1)$ MILP variables in total; the activation variables for flows and jumps, which is $n(|J|+|F|)$; the additional Boolean variables for indicating activated disjuncts in conditions, which is $2m'+nm'(|J|+2|F|)$ as shown in Table~\ref{tab:ind_num}. Thus, we know the number of all the variables is in $\mathcal{O}(n(|V|+m'(|J|+|F|))$, where $\mathcal{O}$ is the asymptotic notation.

\begin{table}[h]\small
\caption{\small Numbers of indicator variables in C1-C10.}
\label{tab:ind_num}

\begin{tabular}{|c|c|c|c|c|c|c|}
\hline
 C1   & C2 & C3 & C4 or C5 & C6 & C7 or C8 & C9 or C10 \\ \hline
$2m'$ & $0$ & $nm'|J|$ & $0$ & $nm'|F|$ & $nm'|F|$ & $0$ \\ \hline
\end{tabular}

\end{table}

\begin{table}[h]\small
\caption{\small Numbers of linear constraints in C1-C10.}
\label{tab:cons_num}

\begin{tabular}{|c|c|c|c|c|c|c|}
\hline
 C1 & C2 & C3 or C4 & C5 & C6 & C7 or C8 & C9 or C10 \\ \hline

 $2m$ & $nK$ & $nm|J|$ & $n|J|$ & $2nm|F|$ & $nm|F|$ & $n|F|$ \\ \hline
\end{tabular}

\end{table}
We show the total numbers of constraints in C1-10 in Table~\ref{tab:cons_num}. The total number of all these constraints is $2m + n(K + (2m+1)|J| + (3m+1)|F|)$. As $K \leq |F|$, we know the number of total constraints is in $\mathcal{O}(nm(|J|+|F|))$. At each time step, the constraint number increase linearly with the product of the number of operators $|J|+|F|$ and the number of disjuncts in a condition $m$.

\subsection{Additional Techniques for Speeding up}
\label{sec:milp:additional}
To accelerate MILP solving, we add additional constraints to encode conflicting operators that cannot happen together or in sequence. While constraints C3, C4, C6, and C7 already prevent these conflicting operators from happening, additional constraints can help a MILP optimizer to effectively prune state space and thus reduce total runtimes.

One type of additional constraints is about the flows with mutually exclusive conditions. Let $f, f' \in F$ be two flows whose differential equations scope on different continuous variable sets $X_k, X_{k'}$, and $\texttt{cond}_f$ and $\texttt{cond}_{f'}$ be there conditions. If $\texttt{cond}_f \land \texttt{cond}_{f'}$ is always $\false$, which means their specified states are totally disjoint, we add constraint $p^f_i + p^{f'}_i  \leq 1$ for every $i \in \{0,1,..,n-1\}$, which ensures at most one of these flows can be activated at every step. The total number of the added constraints is $nm_f$, where $m_f$ is the number of conflicting flow pairs. Note that these constraints are redundant, which are already encoded by C6, but could help a MILP optimizer to easily identify conflicting operators without further checking the complex constraints in C6.

Another type of additional constraints is about subsequent conflicting operators. Let $o, o' \in J \cup F$ be two operators, and $\texttt{post}_o^Q$ and $\texttt{pre}_{o'}^Q$ be the possible internal states after taking $o$, and the possible internal states before $o'$, respectively. We have $\texttt{pre}_{o'}^Q = \texttt{cond}_{o'}^Q$ regardless of $o'$ is a jump or a flow, where $\texttt{cond}_{o'}^Q$ is the condition of $o'$ over internal states $Q$. On the other hand, $\texttt{post}_{o}$ can be different given different operator types: if $o$ is a flow, $\texttt{post}_{o}$ is still $\texttt{cond}_{o}^Q$ since flow $o$ needs this condition to hold during its happening; if $o$ is a jump, we set $\texttt{post}_{o} = \{ \texttt{eff}_o(q,u) \ | \ \texttt{cond}^Q_o(q) = \true \text{ and } (q,u) \in V \}$. If $\texttt{post}_o \land \texttt{pre}_{o'}$ is always $\false$, we know they cannot happen in sequence. Thus, we add constraint $p_i^o + p_{i+1}^{o'} \leq 1$ for every $i \in \{0,1,..,n-2\}$. The total number of the added constraints is $(n-1)m_o$, where $m_o$ is the number of conflicting operator pairs that cannot happen in sequence.

\section{Temporally Concurrent Goals}\label{section:multiple}
In this section, we extend our method to handle a set of temporally concurrent goals by compiling them into our hybrid automata with a certain final goal. There are various types of specifications to represent desired system behaviors over continuous time in both planning and control, such as STL (Signal Temporal Logic) \cite{maler2004monitoring}, and Qualitative State Plan (QSP) \cite{hoffmann2001ff}. While the former formalism has a more expressive syntax by using formal logic, QSP is a well-known specification used in planning and is more suitable to the applications we consider in this paper, which specifies a set of tasks to complete as well as the temporal bounds between their starts and ends. We introduce our method to deal with QSP in this section and present related experiments in Section~\ref{section:results:delivery}. We view exploring the methods and applications related to STL as our future work.


\begin{figure}[ht!]
\centering
\includegraphics[width=0.95\columnwidth]{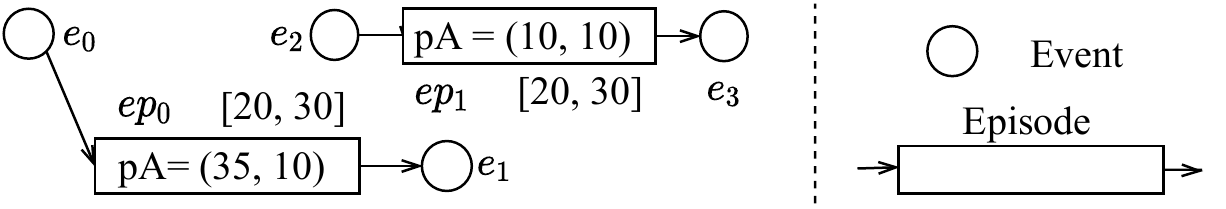}
\caption{\small QSP example of four events and two episodes. Episode $ep_0$ constrain the astronaut to stay at the initial state between $20$ and $30$ minutes right after the mission begins; Episode $ep_1$ constrain the astronaut to stay at the charge station between $20$ and $30$ minutes sometime during the mission.}
\label{fig:qsp}
\end{figure} 

A $QSP$ is a tuple $qsp = \langle EV, EP \rangle$: $EV$ is a set of events, and $e_0 \in EV$ is the initial event that represents the mission begins; $EP$ is a set of episodes. Each episode $ep =\langle e^\vdash e^\dashv, lb, ub, \texttt{cond}  \rangle$ is associated with start and end events $e^\vdash, e^\dashv \in EV$, a duration bound $[lb, ub]$, and a condition $\texttt{cond}$. For each $e \in EV$, we denote $\{ep \in EP \ | \ e = ep.e^\vdash \}$ as $\texttt{starting}(e)$ and $\{ep \in EP \ | \ e = ep.e^\dashv \}$ as $\texttt{ending}(e)$. An example of QSP is given in Figure~\ref{fig:qsp}.

A schedule $s: EV \rightarrow \reals^\geq$ to $qsp$ is a function that maps $e \in EV$ to a non-negative real value such that (1) $s(e_0) = 0$; and (2) $lb_i \leq s(e_i^\vdash) - s(e_i^\dashv) \leq ub_i$ for every $ep_i \in EP$. We say a trajectory $\xi$ satisfies $qsp$ if there exists a schedule $s$ such that for every $ep_i \in EP$, $\texttt{cond}_i(\xi(t)) =\true$  when $(s(e^\vdash_i) < t < s(e^\dashv_i))$.

Given a hybrid automaton $\hybrid = \langle V = (Q \cup E), \init, \goal, J, F \rangle $ and a QSP $qsp = \langle EV, EP \rangle$, we compile this QSP into the original automaton as described below, and the runs of the obtained new automaton $\hybrid'$ respect both $\hybrid$ and $qsp$. We denote this new automaton as $\hybrid' = \langle V' = (Q'\cup E), \init', \goal', J', F'\rangle$.

First, we make a clock variable $c_{ep}$ with domain $[-2, \infty)$ for each episode $ep \in EP$. While $c_{ep} = -1$ means $ep$ has not started, $c_{ep} = -2$ means $ep$ has been achieved. When $ep$ is happening, $c_{ep} \geq 0$. Thus, the continuous state variables of $\hybrid'$ is $Q' = Q \cup C$ and $C  = \{c_{ep} \in [-2, \infty) \ | \ ep \in EP \}$. Since all the episodes have not started in the beginning except the episodes started by initial event $e_0$, the new initial state is $\init' = \init \cup \{(c = -1) \ | \ c \in C/\texttt{starting}(e0)\} \} \cup \{c = 0 \ | \ ep \in \texttt{starting}(e_0)\}$. As all the episodes should be achieved eventually, the new goal is $\goal' = \goal \cup \{(c = -2) \ | \ c \in C \}$.

To describe that clock variables reset at events, we add a set of jumps $J_{EV}$, and $J' = J \cup J_{EV}$.  For each event $e \in EV$, there is a jump $j_e \in J_{EV}$ with the following condition: $\{(c_{ep} = -1) \ | \ ep \in \texttt{starting}(ep)\}$, which ensures event $e$ has not happened before, and  $\{ (lb(ep) \leq c_{ep} \leq ub(e)) \ | \ ep \in \texttt{ending}(e) \}$, which shows $e$ should end only when all the ended episodes has lasted for a proper duration with respect to their temporal bounds. The effects $\{(c_{ep} = 0) \ | \ ep \in \texttt{starting}(e)\}$ and $\{(c_{ep} = -2) \ | \ ep \in \texttt{ending}(e)\}$ capture the clock variable resets for started episodes and ended episodes, respectively.

To force the condition is imposed and its clock variable clicks when an episode is happening, we have a flow $f_{ep}^1$ for each clock variable $c \in C$. This flow has differential equation $\dot c_{ep} = 1$ and conditions $(c_{ep} \geq 0) \cup \texttt{cond}_{ep}$. We also have $f_{ep}^0$ with differential equation $\dot c_{ep} = 0$ and condition $(c_{ep} \leq -1)$ to represent that episode $ep$ is not happening. Thus, the new flows are $F' = F \cup F_{EP}$ and $F_{EP} =  \{ f_{ep}^0 \ | \ ep \in EP \} \cup \{ f_{ep}^1 \ | \ ep \in EP \}$.


\section{Experimental Results}\label{section:results}

To demonstrate the capabilities of our method, we ran our MILP encoding with Gurobi 9.0.1, which is highly optimized and leverages multiple processor cores, and benchmarked against Scotty \cite{fernandez2018scottyactivity} on the Mars transportation domains with different initial setups, the air refueling domains with different numbers of UAVs taking photos in different numbers of regions, and the truck-and-drone delivery domains with different numbers of trucks, drones, and packages. All experiments were run on a 3.40GHZ 8-Core Intel Core i7-6700 CPU with 36GB RAM with a runtime limit of $600$s. At the end of this section, we also discuss the sizes of these MILP encodings.

\begin{figure}[t]
    \centering
    \subfloat[The rover directly picks up and delivers the astronaut to the destination.]{\includegraphics[frame,trim={0.7cm 1.2cm 0.2cm 0.6cm},clip,width=0.45\columnwidth]{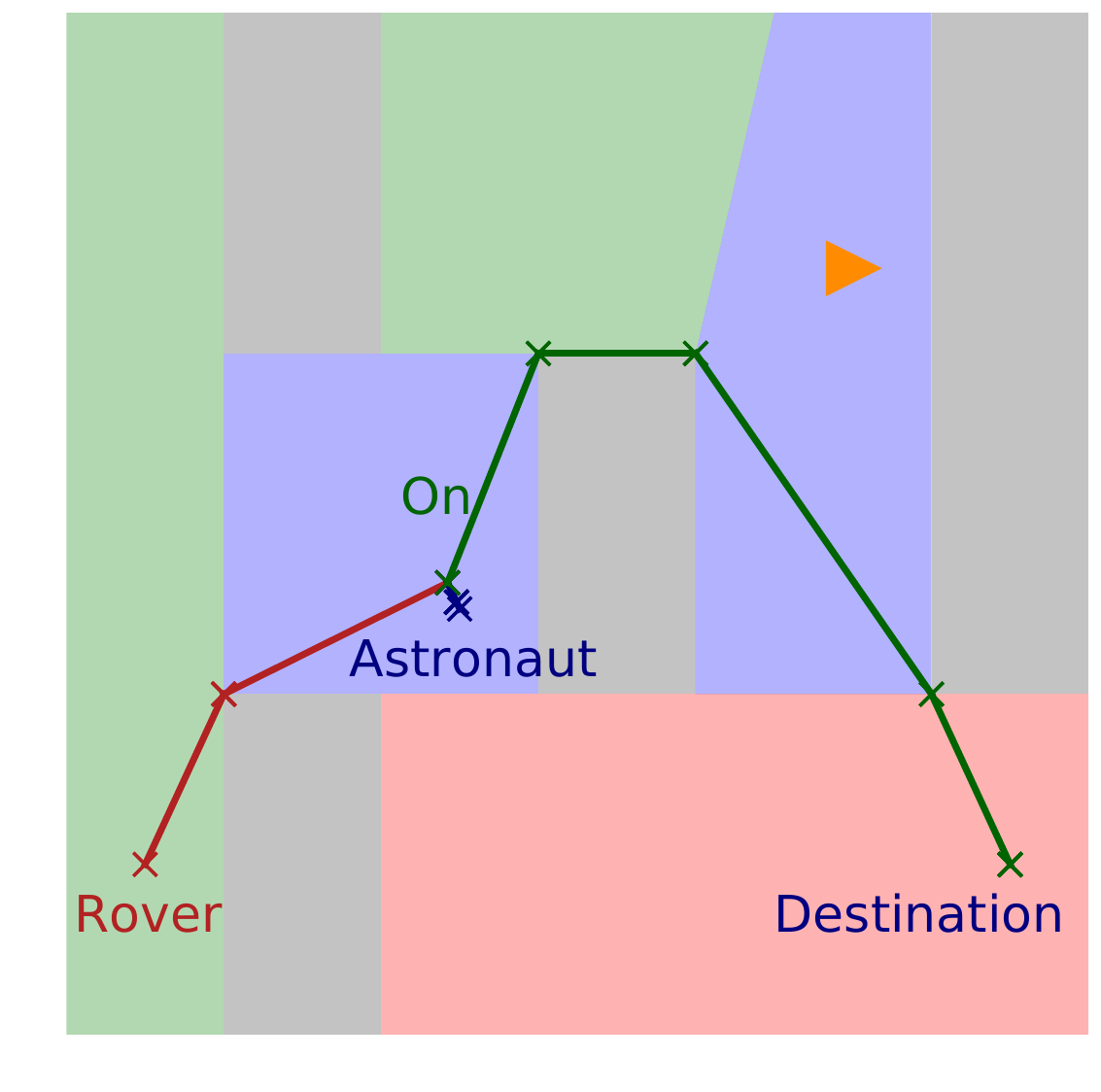}}\
    \,
    \subfloat[The rover does not have enough battery for the trip or going to the charge station, and the astronaut has to walk.]{\includegraphics[frame,trim={0.7cm 1.2cm 0.2cm 0.6cm},clip,width=0.45\columnwidth]{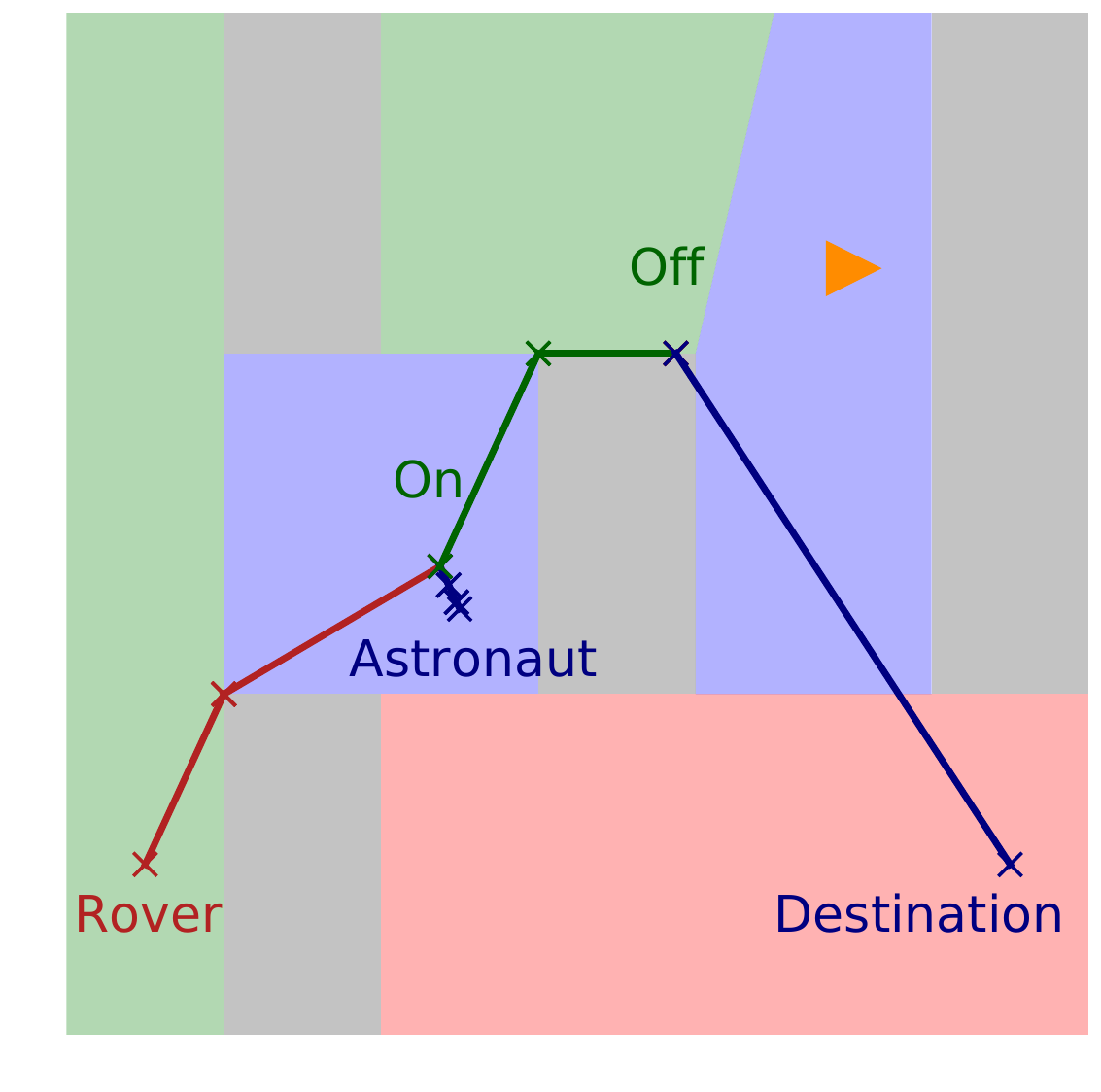}}
    \,
    \subfloat[The rover picks up and delivers the astronaut but  has to recharge during the trip.]{\includegraphics[frame,trim={0.7cm 1.2cm 0.2cm 0.6cm},clip,width=0.45\columnwidth]{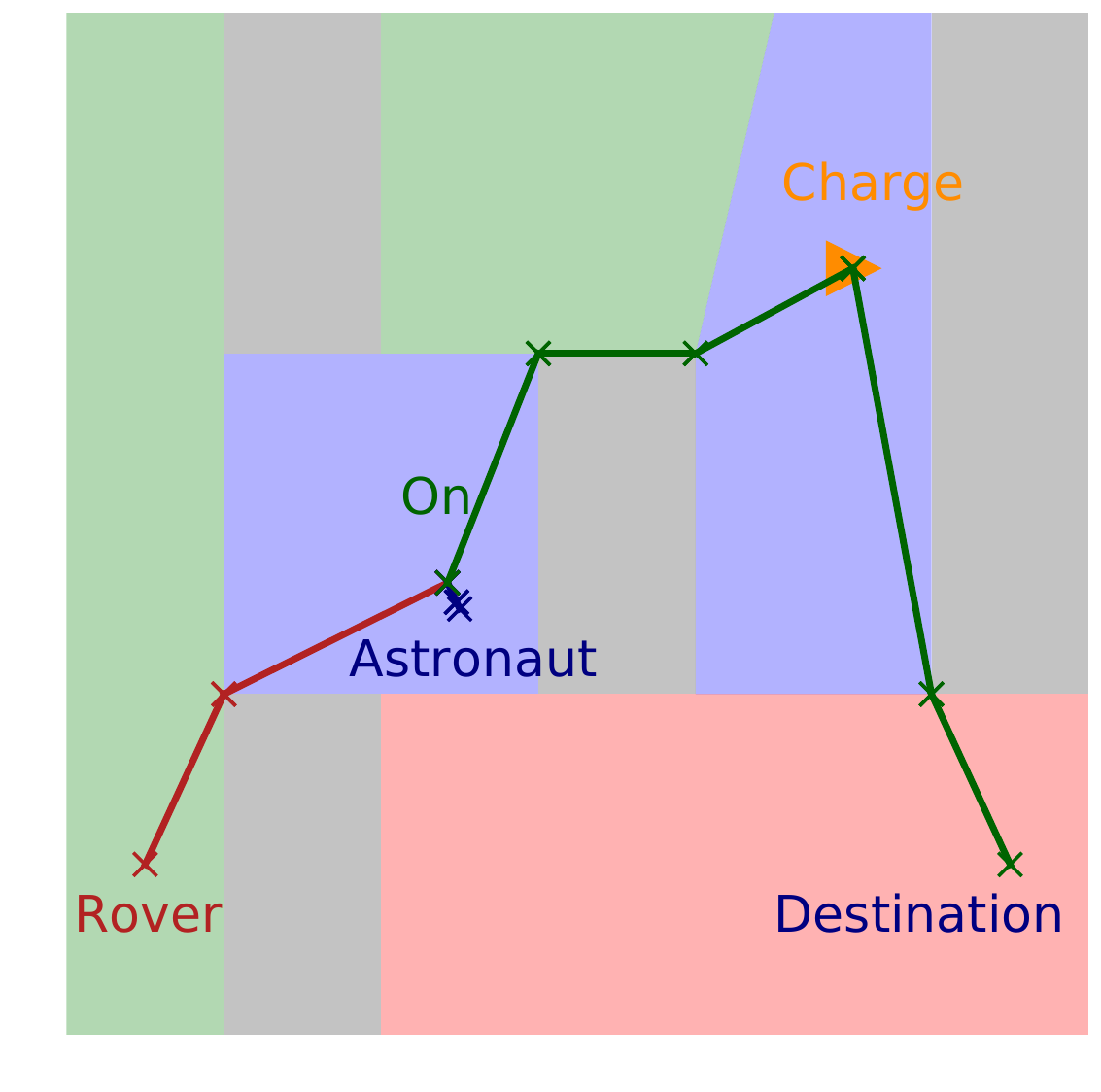}}
    \,
    \subfloat[The rover picks up and delivers the astronaut after recharging.]{\includegraphics[frame,trim={0.7cm 1.2cm 0.2cm 0.6cm},clip,width=0.45\columnwidth]{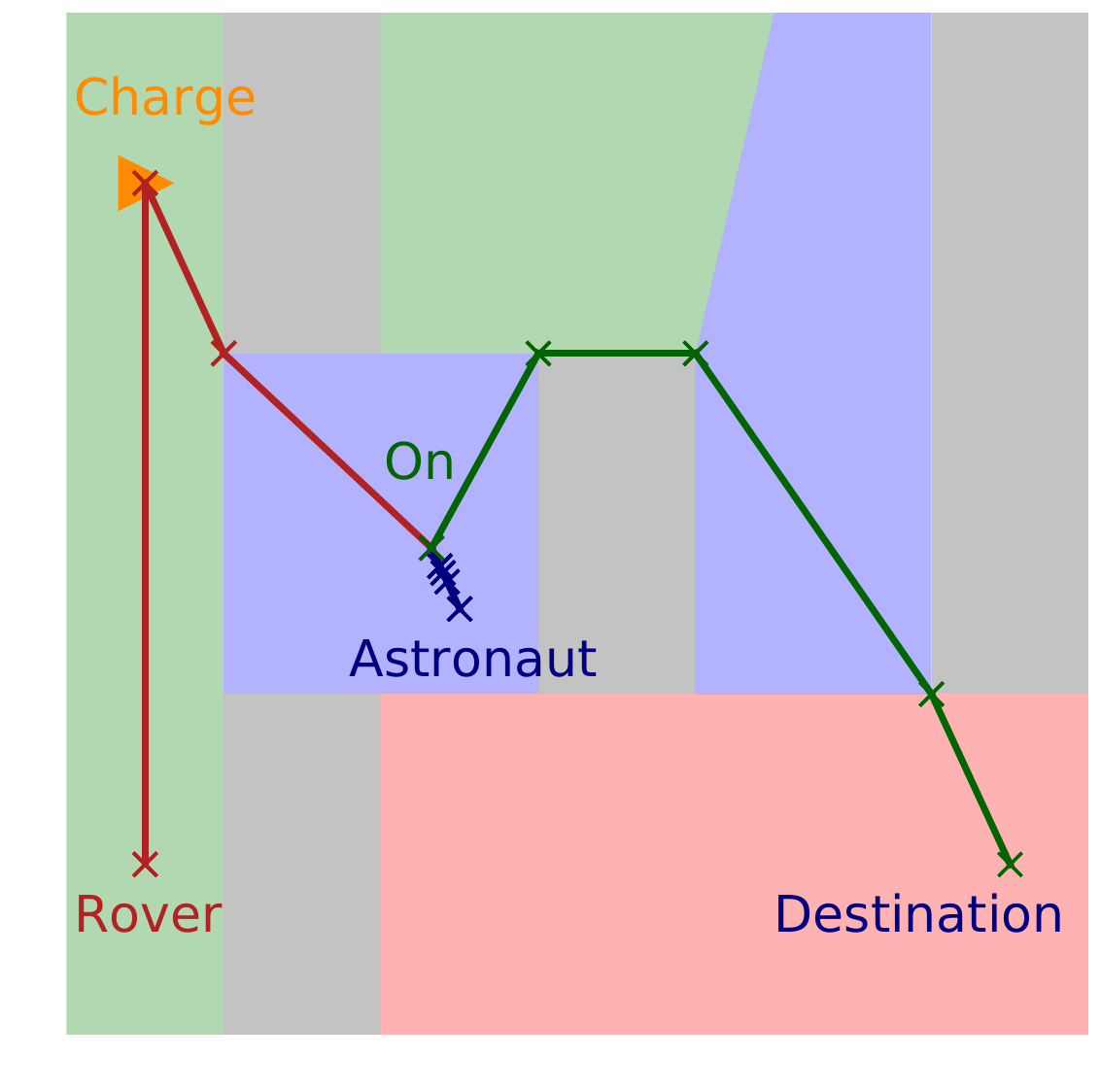}}
    \caption{\small Mars transportation examples with different initial battery levels and charge station locations: the charge station is marked as $\triangleright$; the forbidden areas, mountain, ground, and basin are in gray, red, green, and blue, respectively; the route of the rover is in red and starts from the bottom left; the route of an astronaut walking is in blue, and its goal destination is at the bottom right; the route of the astronaut taking the rover to the destination is in green.}
    \label{fig:mars_execution}

\end{figure}

\begin{table*}[t]
\caption{\small Experimental results of twelve domains. The three numbers after each delivery domain name are the numbers of trucks, drones, and packages, respectively; $t$: the total runtime in seconds; $g$: the makespan of the returned solution; $t_1$: the runtime to find the first solution; $g_1$: the makespan of the first solution; $t_*$: the runtime to first find the solution that is finally returned; $n$: the number of actions; $\# \mathcal{V}_C$, $\# \mathcal{V}_I$, \#C: the numbers of continuous variables, integer variables, and constraints in our MILP encoding; $\# \mathcal{V}_C'$, $\# \mathcal{V}_I'$, \#C': the numbers of continuous variables, integer variables, and constraints in the presolved MILP models. }

\label{tab:results}
\begin{tabular}{|c||c|c||c|c|c|c|c||c|c|c|c|c|c|c|}
\hline
\multirow{2}{*}{Domain} & \multicolumn{2}{c|}{Scotty} & \multicolumn{12}{c|}{MILP Encoding} \\ \cline{2-15} 
 & $t$ & $g$ & $t$ & $g$ & $t_1$ & $g_1$ & $t_*$ & $n$ & $\# \mathcal{V}_C$ & $\# \mathcal{V}_C'$ & $\# \mathcal{V}_I$ & $\# \mathcal{V}_I'$ & \#C & \#C' \\ \hline
Mars (a) & \textless{1} & 35 & \textless{1} & 4.6 & \textless{1} & 35 & \textless{1} & 6 & 65 & 54 & 66 & 34 & 1003 & 527 \\ \hline
Mars (b) & \textless{1} & 35 & \textless{1} & 25.2 & \textless{1} & 35 & \textless{1} & 6 & 65 & 54 & 66 & 34 & 1003 & 527 \\ \hline
Mars (c) & \textless{1} & 35 & 5.0 & 5.2 & \textless{1} & 35 & \textless{1} & 9 & 95 & 84 & 99 & 58 & 1501 & 884 \\ \hline
Mars (d) & \textless{1} & 35 & 3.9 & 6 & \textless{1} & 35 & \textless{1} & 9 & 95 & 84 & 99 & 58 & 1501 & 915  \\ \hline
Air (a) & 4 & 91 & \textless{1} & 43.7 & \textless{1} & 82.7 & \textless{1} & 7 & 75 & 58 & 80 & 22 & 897 & 423 \\ \hline
Air (b) & 17 & 163 & 2.4 & 55.9 & \textless{1} & 103.9 & 1.5 & 12 & 125 & 109 & 185 & 117 & 2443 & 1826 \\ \hline
Air (c) & \textgreater{600} & NA & \textgreater{600} & 84.3 & 1.7 & 177.8 & 70.9 & 24 & 245 & 232 & 610 & 478 & 11117 & 9274 \\ \hline
Air (d) & \textgreater{600} & NA & \textgreater{600} & 40.7 & 0.6 & 108.2 & 154.4 & 18 & 278 & 260 & 566 & 450 & 13692 & 11181 \\ \hline
Delivery (a) (1,1,1) & NA & NA & 8.9 & 36 & \textless{1} & 240 & 3.3 & 7 & 99 & 91 & 339 & 240 & 6449 & 3347 \\ \hline
Delivery (b) (1,2,2) & NA & NA & 13.0 & 12 & \textless{1} & 120 & 3.3 & 7 & 129 & 115 & 402 & 308 & 7610 & 4325 \\ \hline
Delivery (c) (2,4,4) & NA & NA & \textgreater{600} & 132 & \textless{1} & 780 & 219.6 & 11 & 291 & 267 & 1128 & 999 & 28538 & 20804 \\ \hline
Delivery (d) (2,4,4) & NA & NA & \textgreater{600} & 240 & 1.9 & 960 & 14.5 & 11 & 339 & 319 & 1172 & 1029 & 34130 & 25583 \\ \hline
\end{tabular}

\end{table*}

\subsection{Mars Transportation Domain}
The Mars transportation domains involve reasoning over obstacle avoidance and battery consumption under different terrains, such that the astronaut can reach the destination with the help of the rover in the shortest time. A map consists of a set of regions, and each region is a polygon associated with a terrain type. A region can be of the forbidden area, mountain, ground, and basin, which follows the terrains in Figure~\ref{fig:mars}. Driving a rover in different terrains has different velocity limits and energy consumption rates: driving in the mountains should be limited to 10km/h, and the battery consumption rate is 3unit per hour; the velocity limit and the consumption rate in the basin is 30km/h and 2unit/h, and those are 50km/h and 2unit/h for the ground. Walking in these three terrains is 2km/h. On the map, while we fix the initial locations and destinations for astronauts and rovers, we vary the locations of charge stations and the initial battery levels in the four different examples in Figure~\ref{fig:mars_execution}.

In Figure~\ref{fig:mars_execution}, the forbidden areas, mountain, ground, and basin are in gray, red, green, and blue, respectively. As we can see, while the rover starts from its initial location at the bottom left and traverses through mountains and basins to arrive at the destination, the astronaut walks towards the rover and joins the ride. It is interesting to notice that the rover chooses the upper route since traversing the lower mountain area costs more time and energy. While the battery is enough for the rover to complete the route in (a), it is insufficient in (b) and (c). In Figure~\ref{fig:mars_execution}(c), the rover carries the astronaut to the charge station and then continues the mission after getting enough battery. In  Figure~\ref{fig:mars_execution}(b), the rover battery is too low and even insufficient for the trip to the charge station. Thus, the astronaut gets off and walks from the closest location to the destination after draining the battery. In Figure~\ref{fig:mars_execution}(d), the rover also gets recharged, but it happens before picking up the astronaut due to different charge station locations.

Table~\ref{tab:results} shows that our method is able to find the optimal solution and prove its optimality for all these four domains. While Scotty can find a consistent solution $g = 35$ within one second $t<1$, our method can also find such a solution $g_1 = 35$ within one second $t_1<1$. In this solution, the astronaut directly moves to the destination without the help of the rover, which only uses one action but takes a very long time. While Scotty stops after finding this solution, our method keeps searching for better solutions and finds the optimal solution roughly within one second $t_* < 1$. These solutions are then proved to be optimal and returned as Gurobi exhausts the solution space. Thus, our method is able to quickly find a consistent solution and an optimal solution for the Mars transportation domains.

\begin{figure}[]
    \centering
    \subfloat[The UAV takes photos for three regions and does not need refueling.]{\includegraphics[frame,trim={0cm 0cm 0cm 0cm},clip,width=0.45\columnwidth]{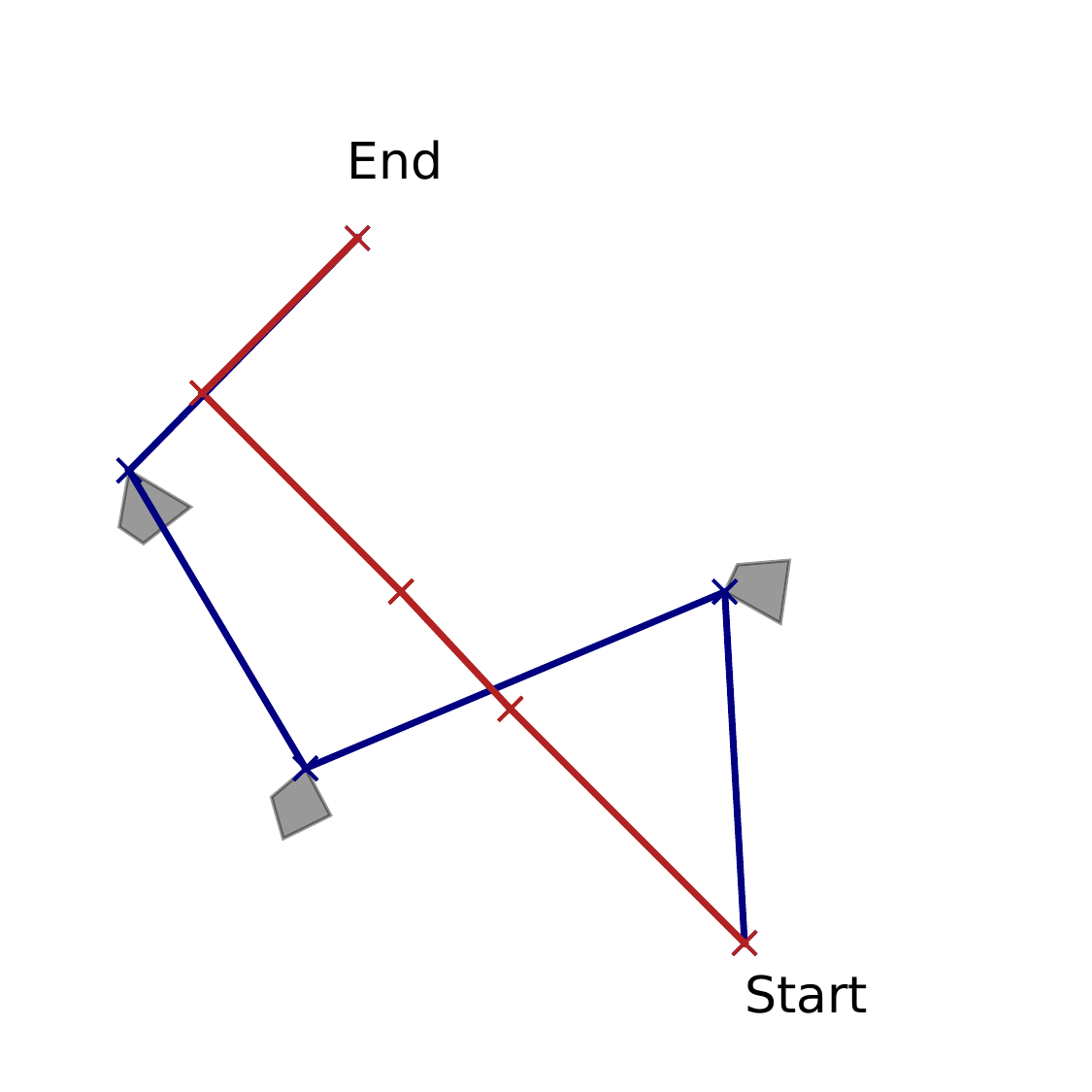}}\
    \,
    \subfloat[The UAV takes photos for four regions and refuels once along the route.]{\includegraphics[frame,trim={0cm 0cm 0cm 0cm},clip,width=0.45\columnwidth]{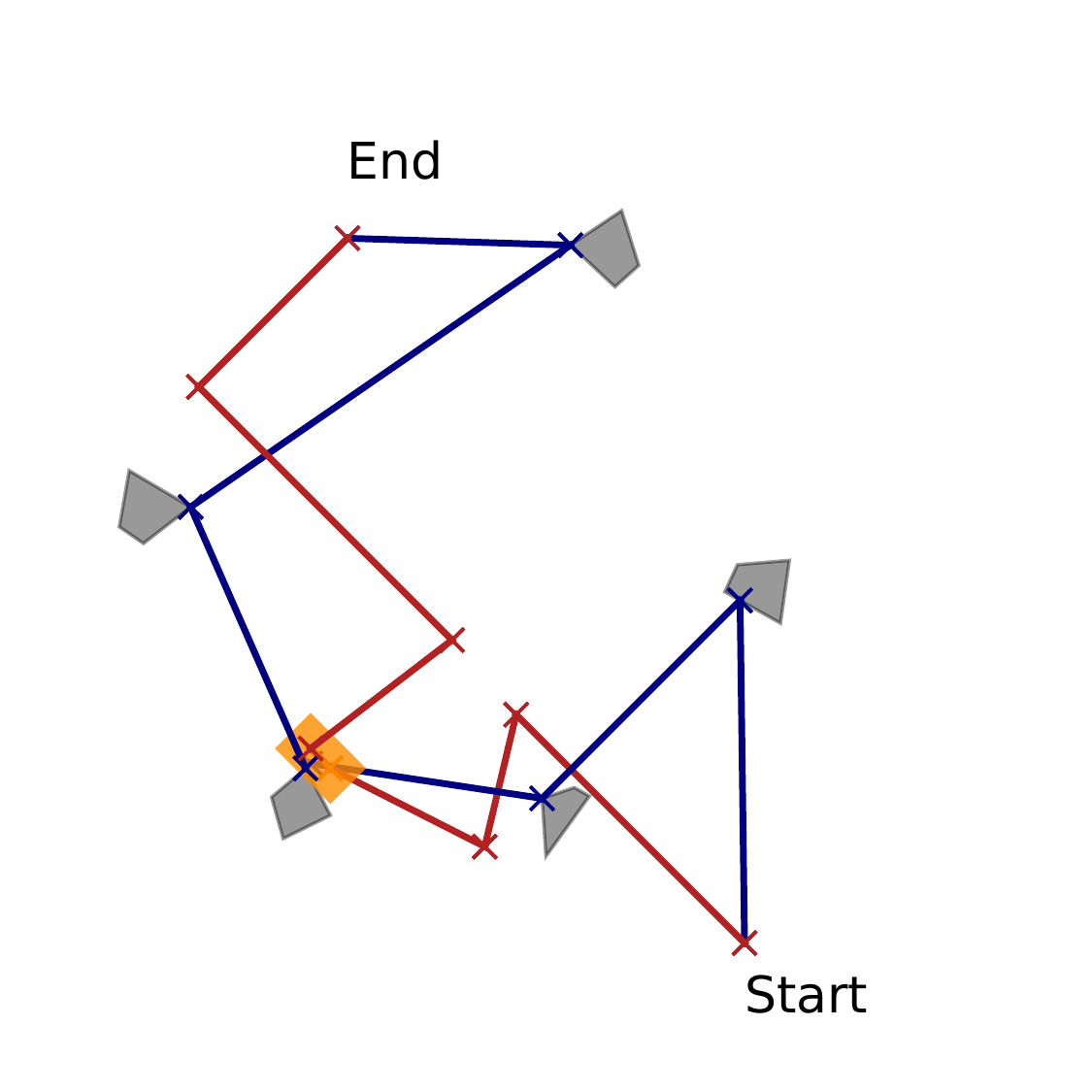}}
    \,
    \subfloat[The UAV takes photos for ten regions and refuels twice along the route.]{\includegraphics[frame,trim={0cm 0cm 0cm 0cm},clip,width=0.45\columnwidth]{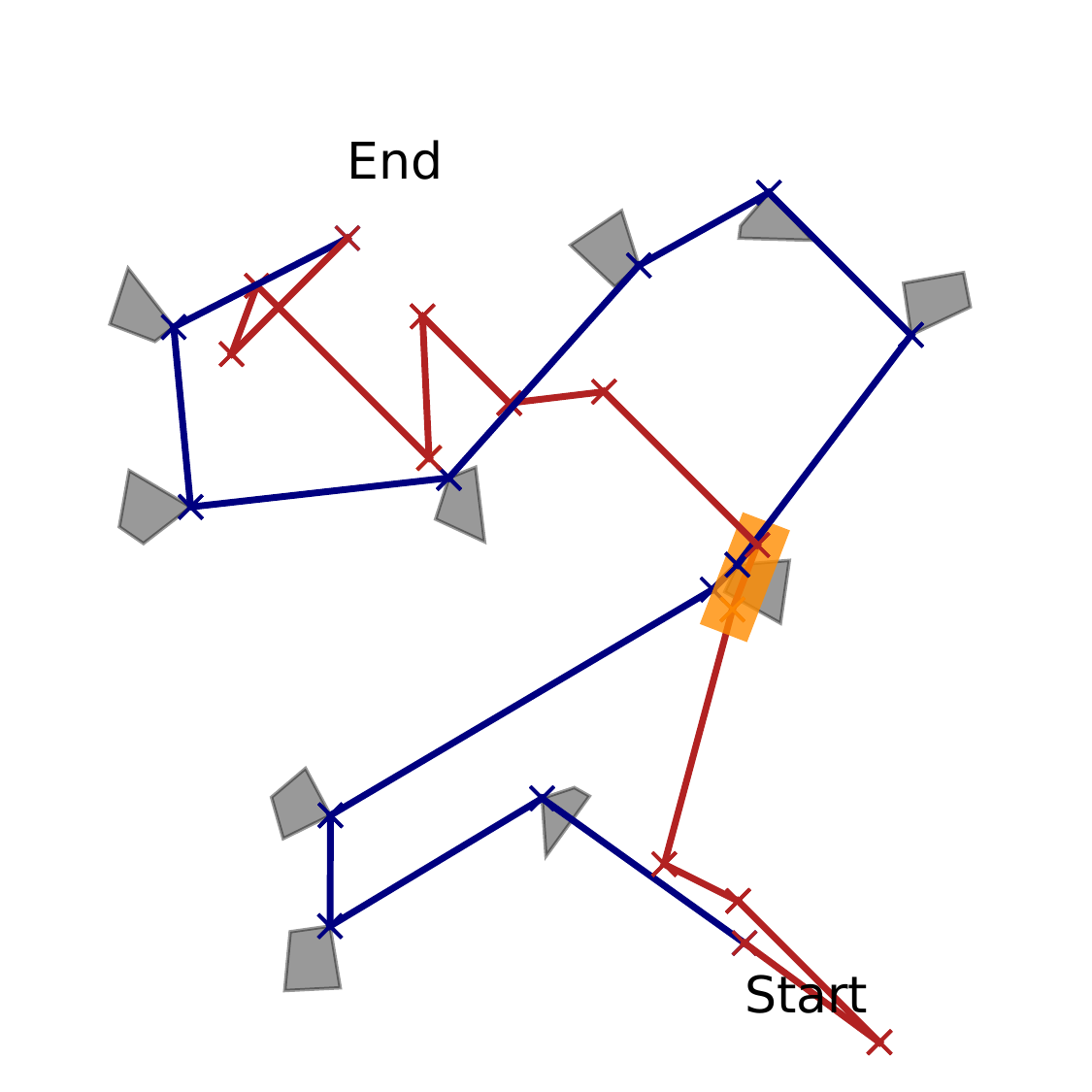}}
    \,
    \subfloat[Two UAVs take photos for eight regions along two different routes. While one UAV does not need refueling, the other one refuels once.]{\includegraphics[frame,trim={0cm 0cm 0cm 0cm},clip,width=0.45\columnwidth]{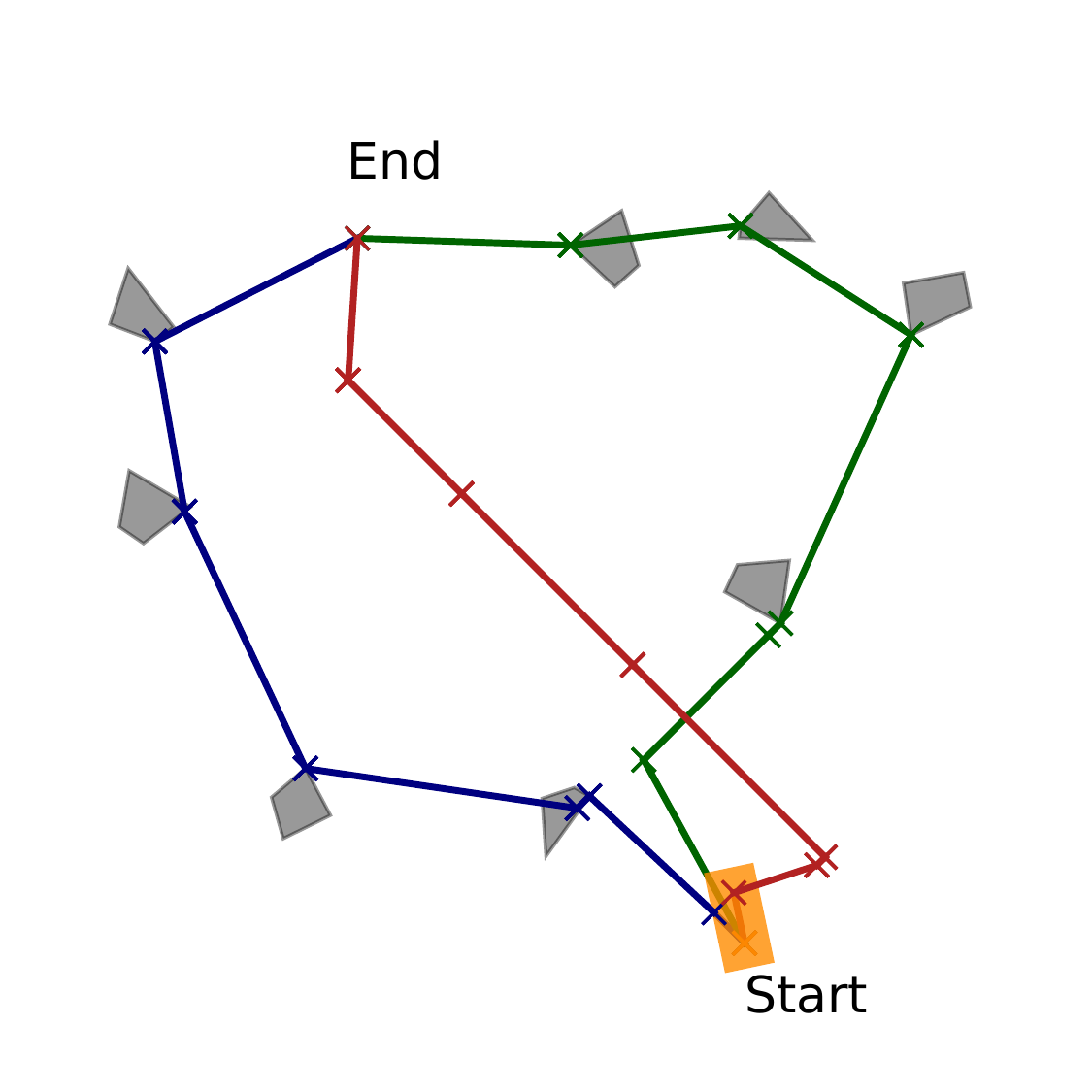}}
    \caption{\small \small Air refueling examples with different numbers of UAVs and regions to take photos:  the regions for taking photos are gray polygons; all the examples consider one UAV (blue) and one tank plane (red) except example (d), which has an additional UAV (green); all the UAVs (blue) are fueled up (i.e., 100 units) in the beginning except the second UAV (green) in domain (d), whose fuel is 10 units. When the plane is refueling, the routes are makred in yellow.}
    \label{fig:air_execution}
\end{figure}

\subsection{Air Refueling Domain} \label{section:results:air}
In this domain, autonomous Unmanned Aerial Vehicles (UAVs) need to take pictures of several regions before landing at the destination location. Since a UAV has limited fuel, it needs to refuel in-air from a tanker plane. This problem is difficult since it requires reasoning on the optimal ordering of visiting all the regions and also coordinating the UAVs and the tank planes to take necessary refueling. When multiple UAVs are in a mission, we should also effectively dispatch the photo-taking tasks such that the makespan is minimized. The maximum velocity of the tank plane is $20m/s$. While flying, UAVs can fly with the velocity up to $30m/s$, and the fuel decreases at $2\text{unit}/s$. Refueling requires the distances between UAVs and tanks planes to be less than $10m$. When an UAV is refueling, the maximum allowable velocity is $5m/s$, and the fuel increases at $10\text{unit}/s$. While the tank capacity of UAVs is $100$ units, we assume the tank plane has enough fuel during missions.

We experimented with this domain on four examples with different numbers of regions and UAVs, as shown in Figure~\ref{fig:air_execution}. The UAVs and the plane start from the same spot and should arrive at the same destination. While there is only one UAV in the examples (a), (b), and (d), we add another UAV in example (d). All the examples only have one tank plane. While our method succeeds in finding feasible solutions in two seconds for all the examples, Scotty spends much more time on (a) and (b) and fails to solve the other two examples within $10$ minutes, which require more complex coordination on visiting a larger number of regions. It is interesting to note that our first solutions are already better than the Scotty solutions, and the makespans of our final solutions are mostly half of those of Scotty. This is because the delete-relaxation heuristics in Scotty misguided its greedy search when energy resources (i.e., fuel) are in this domain, which prevents Scotty from being effective or efficient in this domain.

\subsection{Truck-and-Drone Delivery Domain}\label{section:results:delivery}

In this domain, we consider a fleet of delivery trucks, each equipped with a couple of drones, and the drone and truck both make deliveries to the customers. While trucks can travel between depots through highways or roads, the drone can fly freely in obstacle-free regions or land on trucks to take a ride. When trucks are driving on the road, they should follow the minimum and maximum speed limits as well as the directions, which prevents trucks from violating the traffic rules such as making U-turns on highways. Drones are more flexible, but they are slower, and the travel distance is limited by their battery capacity. In this domain, we look for a plan to deliver all the packages in the shortest time. Figure~\ref{fig:truck_drone} shows an example of the truck-and-drone delivery domains between two depots, in which the two trucks loaded with packages and drones are driving towards each other on a two-way street. Unfortunately, the package destinations are not on the road ahead, and the trucks cannot turn around. A reasonable plan is that the packages are swapped to the other truck by using the drones to cross the street, and then the truck and drone on the other side continue delivery.

\begin{figure}[]
    \centering
    \includegraphics[trim={0.65cm 2.6cm 0.15cm 3cm},clip,width=0.9\columnwidth]{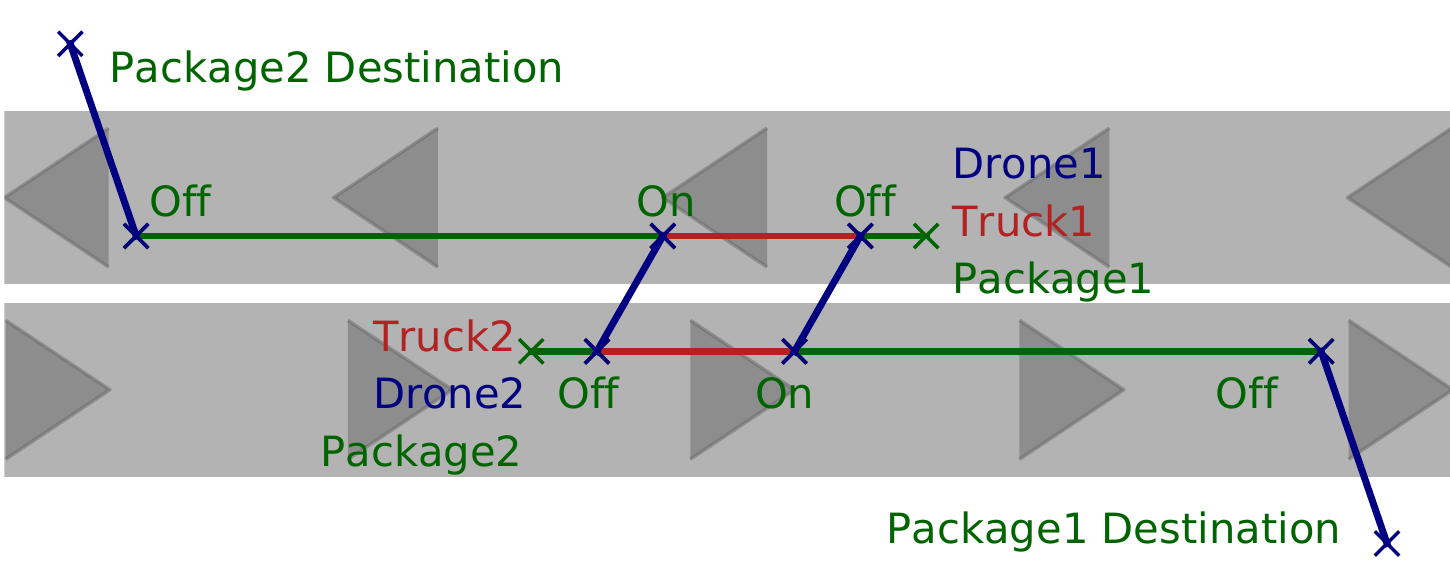}
    \caption{\small Examples of two trucks and two drones delivering two packages:  their initial positions and the package destinations are marked; the truck routes are in red or green and start from the bottom left and the top right, respectively; the routes are in green if the trucks are carrying drones; the routes of drones flying are in blue.}
    \label{fig:truck_drone}
\end{figure}

We test on a map with five depots and ten highways between these depots. Each road is straight and around 10km long with a speed limit of 30-60km/h. The drone can fly with a maximum speed of 5km/h. We assume no obstacle for drones in these examples. The experimental results of four truck-and-drone delivery examples are shown in Table~\ref{tab:results}. While the packages can be delivered at any time in the first three examples, Delivery (d) requires the packages to be delivered within certain time windows specified as a QSP. As we noticed, Scotty does not make progresses to carrying drones to the deport near the drop-off locations and thus fails to solve any of these problems.  It can be seen from the $t_1$ column that our method is able to find the solution very quickly within several seconds. The optimal solutions can also be found in a very short time, $t_*$ for both (a) and (b). In domains (c) and (d), in which we have two trucks, four drones, and four packages, even though it fails to prove the optimality of the incumbent in 10 minutes, their returned solutions largely reduce the makespan of the first returned solutions.

\subsection{MILP Model Study}
Now, we study the MILP models of the three benchmarked domains. Table~\ref{tab:results} shows the numbers of integer variables $\mathcal{V}_I$, continuous variables $\mathcal{V}_C$, and constraints $C$ in our original encoding (Section~\ref{section:milp}), as well as those (i.e., $\mathcal{V}_C'$, $\mathcal{V}_I'$, C') in the model that has been presolved by Gurobi. Gurobi presolves a MILP model by compiling it into a smaller model with equivalent feasible and optimal solutions as the original. As we can see in Table~\ref{tab:results}, the presolved models reduce about $20\%$ continuous variables, $20 \sim 40\%$ integer variables, and $30 \sim 50 \%$ constraints in most examples. We also observed that presolving takes less than $0.1s$ in our experiments.

As the Mars domain does not have many discrete state variables or actions, the numbers of its discrete variables and continuous variables are roughly the same. When it comes to the air refueling domain with more than $8$ regions to visit or the truck-and-drone delivery domain with a large number of discrete variables to indicate trucks are on a certain highway, we observe the number of discrete variables is $2 \sim 5$ times than that of continuous variables. We also note that it is more difficult to find provably optimal solutions in the MILP problems with more integer variables. While Gurobi can find a solution for all the examples in $2s$ even for Delivery (d), which has $1511$ variables and $34130$ constraints, the largest problem we can prove optimality within the runtime limit is Delivery (b), which has $537$ variables and $7610$ constraints.

\section{Conclusions and Future Work}\label{section:conclusion}
In this paper, we presented a mixed discrete-continuous planning approach that fixes the action number of the automaton runs and encodes the corresponding finite-step hybrid planning problem as a MILP. Our complexity analysis shows that the number of the MILP variables and constraints at each step increases linearly with the product of the number of linear constraints involved in each condition and the number of operators and variables. By leveraging the state-of-the-art MILP optimizer Gurobi, our method is able to efficiently find provably optimal or high-quality solutions for challenging mixed discrete-continuous planning problems. This was supported by our experimental results against Scotty on the Mars transportation domains, the air refueling domains, and the truck-and-drone delivery domains.

In this paper, we show how to deal with temporally concurrent goals modeled in QSPs with our MILP approach. For future work, we plan to extend our method to support the full features of STL, which can model more expressive desired systems behaviors. We also would like to explore and solve more real-world applications with this extension.

\subsubsection*{Acknowledgements.} This project was funded by the Defense Advanced Research Projects Agency under Grant Contract No. N16A-T002-0149.

\bibliographystyle{ACM-Reference-Format}
\bibliography{bib}

\end{document}

%% file: defs.tex
\usepackage{xparse}
\newcommand{\bnm}{\begin{newmath}}
\newcommand{\enm}{\end{newmath}}

\newcommand{\bea}{\begin{eqnarray*}}%
\newcommand{\eea}{\end{eqnarray*}}%

\newcommand{\bne}{\begin{newequation}}
\newcommand{\ene}{\end{newequation}}

\newcommand{\bal}{\begin{newalign}}
\newcommand{\eal}{\end{newalign}}

\newenvironment{newalign}{\begin{align}%
\setlength{\abovedisplayskip}{4pt}%
\setlength{\belowdisplayskip}{4pt}%
\setlength{\abovedisplayshortskip}{6pt}%
\setlength{\belowdisplayshortskip}{6pt} }{\end{align}}

\newenvironment{newmath}{\begin{displaymath}%
\setlength{\abovedisplayskip}{4pt}%
\setlength{\belowdisplayskip}{4pt}%
\setlength{\abovedisplayshortskip}{6pt}%
\setlength{\belowdisplayshortskip}{6pt} }{\end{displaymath}}

\newenvironment{newequation}{\begin{equation}%
\setlength{\abovedisplayskip}{4pt}%
\setlength{\belowdisplayskip}{4pt}%
\setlength{\abovedisplayshortskip}{6pt}%
\setlength{\belowdisplayshortskip}{6pt} }{\end{equation}}

\newcounter{ctr}

\newcounter{mytable}
\def\mytable{\begin{centering}\refstepcounter{mytable}}
\def\endmytable{\end{centering}}

\newcounter{myfig}
\def\myfig{\begin{centering}\refstepcounter{myfig}}
\def\endmyfig{\end{centering}}

\newlength{\saveparindent}
\newlength{\saveparskip}
\newcommand{\E}{{\rm I\kern-.3em E}}

\newcommand{\true}{\mathsf{true}}
\newcommand{\false}{\mathsf{false}}

\renewcommand{\eqref}[1]{\mbox{Equation~(\ref{#1})}}

\def \part {part}

\DeclareMathOperator*{\argmin}{argmin}

\renewcommand{\paragraph}[1]{\vspace*{6pt}\noindent\textbf{#1}\;}

\def \blackslug{\hbox{\hskip 1pt \vrule width 4pt height 8pt
    depth 1.5pt \hskip 1pt}}
\def \qed{\quad\blackslug\lower 8.5pt\null\par}

\newcounter{mynote}[section]

\newcommand\ignore[1]{}

\newcounter{rcnote}[section]

\newcounter{mrnote}[section]

\newcounter{fknote}[section]

\newcounter{anote}[section]

\DeclareMathSymbol{\mlq}{\mathord}{operators}{``}
\DeclareMathSymbol{\mrq}{\mathord}{operators}{`'}

\newcommand{\rhf}[2]{R_{f, \gamma}}

\DeclareDocumentCommand{\edist}{o o}{
  \ensuremath{
    \IfNoValueTF{#1}{{d}}{{\sf d}(#1,#2)}
  }
}

\newcommand{\olrk}[1]{\ifx\nursymbol#1\else\!\!\mskip4.5mu plus 0.5mu\left(\mskip0.5mu plus0.5mu #1\mskip1.5mu plus0.5mu \right)\fi}

\NewDocumentCommand{\indseq}{ O{1} O{r} }{{#1}\ldots {#2}}


%% file: macros.tex
\newcommand{\hybrid}{\mathcal{H}}

\newcommand{\val}{{\mathsf{val}}}

\newcommand{\reals}{\mathbb{R}}
\newcommand{\bool}{\mathbb{B}}

\newcommand{\init}{q_\mathtt{Init}}
\newcommand{\goal}{\mathtt{Goal}}

%% file: main.bbl

\begin{thebibliography}{00}


\ifx \showCODEN    \undefined \def \showCODEN     #1{\unskip}     \fi
\ifx \showDOI      \undefined \def \showDOI       #1{#1}\fi
\ifx \showISBNx    \undefined \def \showISBNx     #1{\unskip}     \fi
\ifx \showISBNxiii \undefined \def \showISBNxiii  #1{\unskip}     \fi
\ifx \showISSN     \undefined \def \showISSN      #1{\unskip}     \fi
\ifx \showLCCN     \undefined \def \showLCCN      #1{\unskip}     \fi
\ifx \shownote     \undefined \def \shownote      #1{#1}          \fi
\ifx \showarticletitle \undefined \def \showarticletitle #1{#1}   \fi
\ifx \showURL      \undefined \def \showURL       {\relax}        \fi
\providecommand\bibfield[2]{#2}
\providecommand\bibinfo[2]{#2}
\providecommand\natexlab[1]{#1}
\providecommand\showeprint[2][]{arXiv:#2}

\bibitem[\protect\citeauthoryear{Bit-Monnot, Pulina, and Tacchella}{Bit-Monnot
  et~al\mbox{.}}{2019}]%
        {bit2019cyber}
\bibfield{author}{\bibinfo{person}{Arthur Bit-Monnot}, \bibinfo{person}{Luca
  Pulina}, {and} \bibinfo{person}{Armando Tacchella}.}
  \bibinfo{year}{2019}\natexlab{}.
\newblock \showarticletitle{Cyber-Physical Planning: Deliberation for Hybrid
  Systems with a Continuous Numeric State}. In \bibinfo{booktitle}{{\em
  Proceedings of the International Conference on Automated Planning and
  Scheduling}}, Vol.~\bibinfo{volume}{29}. \bibinfo{pages}{49--57}.
\newblock


\bibitem[\protect\citeauthoryear{Bogomolov, Magazzeni, Minopoli, and
  Wehrle}{Bogomolov et~al\mbox{.}}{2015}]%
        {bogomolov2015pddl+}
\bibfield{author}{\bibinfo{person}{Sergiy Bogomolov}, \bibinfo{person}{Daniele
  Magazzeni}, \bibinfo{person}{Stefano Minopoli}, {and} \bibinfo{person}{Martin
  Wehrle}.} \bibinfo{year}{2015}\natexlab{}.
\newblock \showarticletitle{PDDL+ planning with hybrid automata: Foundations of
  translating must behavior}. In \bibinfo{booktitle}{{\em Twenty-Fifth
  International Conference on Automated Planning and Scheduling}}.
\newblock


\bibitem[\protect\citeauthoryear{Bogomolov, Magazzeni, Podelski, and
  Wehrle}{Bogomolov et~al\mbox{.}}{2014}]%
        {bogomolov2014planning}
\bibfield{author}{\bibinfo{person}{Sergiy Bogomolov}, \bibinfo{person}{Daniele
  Magazzeni}, \bibinfo{person}{Andreas Podelski}, {and} \bibinfo{person}{Martin
  Wehrle}.} \bibinfo{year}{2014}\natexlab{}.
\newblock \showarticletitle{Planning as model checking in hybrid domains}. In
  \bibinfo{booktitle}{{\em Twenty-Eighth AAAI Conference on Artificial
  Intelligence}}.
\newblock


\bibitem[\protect\citeauthoryear{Bryce}{Bryce}{2016}]%
        {bryce2016happening}
\bibfield{author}{\bibinfo{person}{Daniel Bryce}.}
  \bibinfo{year}{2016}\natexlab{}.
\newblock \showarticletitle{A happening-based encoding for nonlinear pddl+
  planning}. In \bibinfo{booktitle}{{\em Workshops at the Thirtieth AAAI
  Conference on Artificial Intelligence}}.
\newblock


\bibitem[\protect\citeauthoryear{Bryce, Gao, Musliner, and Goldman}{Bryce
  et~al\mbox{.}}{2015}]%
        {bryce2015smt}
\bibfield{author}{\bibinfo{person}{Daniel Bryce}, \bibinfo{person}{Sicun Gao},
  \bibinfo{person}{David~J Musliner}, {and} \bibinfo{person}{Robert~P
  Goldman}.} \bibinfo{year}{2015}\natexlab{}.
\newblock \showarticletitle{SMT-Based Nonlinear PDDL+ Planning.}. In
  \bibinfo{booktitle}{{\em AAAI}}. \bibinfo{pages}{3247--3253}.
\newblock


\bibitem[\protect\citeauthoryear{Cashmore, Fox, Long, and Magazzeni}{Cashmore
  et~al\mbox{.}}{2016}]%
        {cashmore2016compilation}
\bibfield{author}{\bibinfo{person}{Michael Cashmore}, \bibinfo{person}{Maria
  Fox}, \bibinfo{person}{Derek Long}, {and} \bibinfo{person}{Daniele
  Magazzeni}.} \bibinfo{year}{2016}\natexlab{}.
\newblock \showarticletitle{A compilation of the full PDDL+ language into SMT}.
\newblock  (\bibinfo{year}{2016}).
\newblock


\bibitem[\protect\citeauthoryear{Cimatti, Clarke, Giunchiglia, and
  Roveri}{Cimatti et~al\mbox{.}}{2000}]%
        {cimatti2000nusmv}
\bibfield{author}{\bibinfo{person}{Alessandro Cimatti}, \bibinfo{person}{Edmund
  Clarke}, \bibinfo{person}{Fausto Giunchiglia}, {and} \bibinfo{person}{Marco
  Roveri}.} \bibinfo{year}{2000}\natexlab{}.
\newblock \showarticletitle{NuSMV: a new symbolic model checker}.
\newblock \bibinfo{journal}{{\em International Journal on Software Tools for
  Technology Transfer\/}} \bibinfo{volume}{2}, \bibinfo{number}{4}
  (\bibinfo{year}{2000}), \bibinfo{pages}{410--425}.
\newblock


\bibitem[\protect\citeauthoryear{Coles, Fox, and Long}{Coles
  et~al\mbox{.}}{2013}]%
        {coles2013hybrid}
\bibfield{author}{\bibinfo{person}{Amanda Coles}, \bibinfo{person}{M Fox},
  {and} \bibinfo{person}{D Long}.} \bibinfo{year}{2013}\natexlab{}.
\newblock \showarticletitle{A hybrid LP-RPG heuristic for modelling numeric
  resource flows in planning}.
\newblock \bibinfo{journal}{{\em Journal of Artificial Intelligence
  Research\/}}  \bibinfo{volume}{46} (\bibinfo{year}{2013}),
  \bibinfo{pages}{343--412}.
\newblock


\bibitem[\protect\citeauthoryear{Coles, Coles, Fox, and Long}{Coles
  et~al\mbox{.}}{2012}]%
        {coles2012colin}
\bibfield{author}{\bibinfo{person}{Amanda~Jane Coles},
  \bibinfo{person}{Andrew~I Coles}, \bibinfo{person}{Maria Fox}, {and}
  \bibinfo{person}{Derek Long}.} \bibinfo{year}{2012}\natexlab{}.
\newblock \showarticletitle{COLIN: Planning with continuous linear numeric
  change}.
\newblock \bibinfo{journal}{{\em Journal of Artificial Intelligence
  Research\/}}  \bibinfo{volume}{44} (\bibinfo{year}{2012}),
  \bibinfo{pages}{1--96}.
\newblock


\bibitem[\protect\citeauthoryear{Della~Penna, Magazzeni, and
  Mercorio}{Della~Penna et~al\mbox{.}}{2012}]%
        {della2012universal}
\bibfield{author}{\bibinfo{person}{Giuseppe Della~Penna},
  \bibinfo{person}{Daniele Magazzeni}, {and} \bibinfo{person}{Fabio Mercorio}.}
  \bibinfo{year}{2012}\natexlab{}.
\newblock \showarticletitle{A universal planning system for hybrid domains}.
\newblock \bibinfo{journal}{{\em Applied intelligence\/}} \bibinfo{volume}{36},
  \bibinfo{number}{4} (\bibinfo{year}{2012}), \bibinfo{pages}{932--959}.
\newblock


\bibitem[\protect\citeauthoryear{Fan, Mathur, Mitra, and Viswanathan}{Fan
  et~al\mbox{.}}{2018a}]%
        {fan2018controller}
\bibfield{author}{\bibinfo{person}{Chuchu Fan}, \bibinfo{person}{Umang Mathur},
  \bibinfo{person}{Sayan Mitra}, {and} \bibinfo{person}{Mahesh Viswanathan}.}
  \bibinfo{year}{2018}\natexlab{a}.
\newblock \showarticletitle{Controller synthesis made real: reach-avoid
  specifications and linear dynamics}. In \bibinfo{booktitle}{{\em
  International Conference on Computer Aided Verification}}. Springer,
  \bibinfo{pages}{347--366}.
\newblock


\bibitem[\protect\citeauthoryear{Fan, Mathur, Mitra, and Viswanathan}{Fan
  et~al\mbox{.}}{2018b}]%
        {FanMMV:CAV2018}
\bibfield{author}{\bibinfo{person}{Chuchu Fan}, \bibinfo{person}{Umang Mathur},
  \bibinfo{person}{Sayan Mitra}, {and} \bibinfo{person}{Mahesh Viswanathan}.}
  \bibinfo{year}{2018}\natexlab{b}.
\newblock \showarticletitle{Controller Synthesis Made Real: Reachavoid
  Specifications and Linear Dynamics}. In \bibinfo{booktitle}{{\em Computer
  Aided Verification}}. \bibinfo{publisher}{Springer International Publishing},
  \bibinfo{pages}{347–366}.
\newblock
\showDOI{%
\url{https://doi.org/10.1007/978-3-319-96145-3\_19}}


\bibitem[\protect\citeauthoryear{Fernandez-Gonzalez, Karpas, and
  Williams}{Fernandez-Gonzalez et~al\mbox{.}}{2017}]%
        {fernandez2017mixed}
\bibfield{author}{\bibinfo{person}{Enrique Fernandez-Gonzalez},
  \bibinfo{person}{Erez Karpas}, {and} \bibinfo{person}{Brian Williams}.}
  \bibinfo{year}{2017}\natexlab{}.
\newblock \showarticletitle{Mixed discrete-continuous planning with convex
  optimization}. In \bibinfo{booktitle}{{\em Thirty-First AAAI Conference on
  Artificial Intelligence}}.
\newblock


\bibitem[\protect\citeauthoryear{Fern{\'a}ndez-Gonz{\'a}lez, Williams, and
  Karpas}{Fern{\'a}ndez-Gonz{\'a}lez et~al\mbox{.}}{2018}]%
        {fernandez2018scottyactivity}
\bibfield{author}{\bibinfo{person}{Enrique Fern{\'a}ndez-Gonz{\'a}lez},
  \bibinfo{person}{Brian Williams}, {and} \bibinfo{person}{Erez Karpas}.}
  \bibinfo{year}{2018}\natexlab{}.
\newblock \showarticletitle{Scottyactivity: Mixed discrete-continuous planning
  with convex optimization}.
\newblock \bibinfo{journal}{{\em Journal of Artificial Intelligence
  Research\/}}  \bibinfo{volume}{62} (\bibinfo{year}{2018}),
  \bibinfo{pages}{579--664}.
\newblock


\bibitem[\protect\citeauthoryear{Filippidis, Dathathri, Livingston, Ozay, and
  Murray}{Filippidis et~al\mbox{.}}{2016a}]%
        {FilippidisDLOM16}
\bibfield{author}{\bibinfo{person}{Ioannis Filippidis},
  \bibinfo{person}{Sumanth Dathathri}, \bibinfo{person}{Scott~C. Livingston},
  \bibinfo{person}{Necmiye Ozay}, {and} \bibinfo{person}{Richard~M. Murray}.}
  \bibinfo{year}{2016}\natexlab{a}.
\newblock \showarticletitle{Control design for hybrid systems with TuLiP: The
  Temporal Logic Planning toolbox}. In \bibinfo{booktitle}{{\em {IEEE}
  Conference on Control Applications}}. \bibinfo{pages}{1030--1041}.
\newblock


\bibitem[\protect\citeauthoryear{Filippidis, Dathathri, Livingston, Ozay, and
  Murray}{Filippidis et~al\mbox{.}}{2016b}]%
        {filippidis2016control}
\bibfield{author}{\bibinfo{person}{Ioannis Filippidis},
  \bibinfo{person}{Sumanth Dathathri}, \bibinfo{person}{Scott~C Livingston},
  \bibinfo{person}{Necmiye Ozay}, {and} \bibinfo{person}{Richard~M Murray}.}
  \bibinfo{year}{2016}\natexlab{b}.
\newblock \showarticletitle{Control design for hybrid systems with TuLiP: The
  temporal logic planning toolbox}. In \bibinfo{booktitle}{{\em 2016 IEEE
  Conference on Control Applications (CCA)}}. IEEE,
  \bibinfo{pages}{1030--1041}.
\newblock


\bibitem[\protect\citeauthoryear{Fox and Long}{Fox and Long}{2003}]%
        {fox2003pddl2}
\bibfield{author}{\bibinfo{person}{Maria Fox} {and} \bibinfo{person}{Derek
  Long}.} \bibinfo{year}{2003}\natexlab{}.
\newblock \showarticletitle{PDDL2. 1: An extension to PDDL for expressing
  temporal planning domains}.
\newblock \bibinfo{journal}{{\em Journal of artificial intelligence
  research\/}}  \bibinfo{volume}{20} (\bibinfo{year}{2003}),
  \bibinfo{pages}{61--124}.
\newblock


\bibitem[\protect\citeauthoryear{Fox and Long}{Fox and Long}{2006}]%
        {fox2006modelling}
\bibfield{author}{\bibinfo{person}{Maria Fox} {and} \bibinfo{person}{Derek
  Long}.} \bibinfo{year}{2006}\natexlab{}.
\newblock \showarticletitle{Modelling mixed discrete-continuous domains for
  planning}.
\newblock \bibinfo{journal}{{\em Journal of Artificial Intelligence
  Research\/}}  \bibinfo{volume}{27} (\bibinfo{year}{2006}),
  \bibinfo{pages}{235--297}.
\newblock


\bibitem[\protect\citeauthoryear{Frehse}{Frehse}{2008}]%
        {frehse2008phaver}
\bibfield{author}{\bibinfo{person}{Goran Frehse}.}
  \bibinfo{year}{2008}\natexlab{}.
\newblock \showarticletitle{PHAVer: algorithmic verification of hybrid systems
  past HyTech}.
\newblock \bibinfo{journal}{{\em International Journal on Software Tools for
  Technology Transfer\/}} \bibinfo{volume}{10}, \bibinfo{number}{3}
  (\bibinfo{year}{2008}), \bibinfo{pages}{263--279}.
\newblock


\bibitem[\protect\citeauthoryear{Frehse, Le~Guernic, Donz{\'e}, Cotton, Ray,
  Lebeltel, Ripado, Girard, Dang, and Maler}{Frehse et~al\mbox{.}}{2011}]%
        {frehse2011spaceex}
\bibfield{author}{\bibinfo{person}{Goran Frehse}, \bibinfo{person}{Colas
  Le~Guernic}, \bibinfo{person}{Alexandre Donz{\'e}}, \bibinfo{person}{Scott
  Cotton}, \bibinfo{person}{Rajarshi Ray}, \bibinfo{person}{Olivier Lebeltel},
  \bibinfo{person}{Rodolfo Ripado}, \bibinfo{person}{Antoine Girard},
  \bibinfo{person}{Thao Dang}, {and} \bibinfo{person}{Oded Maler}.}
  \bibinfo{year}{2011}\natexlab{}.
\newblock \showarticletitle{SpaceEx: Scalable verification of hybrid systems}.
  In \bibinfo{booktitle}{{\em International Conference on Computer Aided
  Verification}}. Springer, \bibinfo{pages}{379--395}.
\newblock


\bibitem[\protect\citeauthoryear{Garrett, Lozano-P{\'e}rez, and
  Kaelbling}{Garrett et~al\mbox{.}}{2015}]%
        {garrett2015ffrob}
\bibfield{author}{\bibinfo{person}{Caelan~Reed Garrett},
  \bibinfo{person}{Tom{\'a}s Lozano-P{\'e}rez}, {and}
  \bibinfo{person}{Leslie~Pack Kaelbling}.} \bibinfo{year}{2015}\natexlab{}.
\newblock \showarticletitle{FFRob: An efficient heuristic for task and motion
  planning}.
\newblock In \bibinfo{booktitle}{{\em Algorithmic Foundations of Robotics XI}}.
  \bibinfo{publisher}{Springer}, \bibinfo{pages}{179--195}.
\newblock


\bibitem[\protect\citeauthoryear{Garrett, Lozano-P{\'e}rez, and
  Kaelbling}{Garrett et~al\mbox{.}}{2017}]%
        {garrett2017sample}
\bibfield{author}{\bibinfo{person}{Caelan~Reed Garrett},
  \bibinfo{person}{Tom{\'a}s Lozano-P{\'e}rez}, {and}
  \bibinfo{person}{Leslie~Pack Kaelbling}.} \bibinfo{year}{2017}\natexlab{}.
\newblock \showarticletitle{Sample-Based Methods for Factored Task and Motion
  Planning.}. In \bibinfo{booktitle}{{\em Robotics: Science and Systems}}.
\newblock


\bibitem[\protect\citeauthoryear{Garrett, Lozano-Perez, and Kaelbling}{Garrett
  et~al\mbox{.}}{2018}]%
        {garrett2018ffrob}
\bibfield{author}{\bibinfo{person}{Caelan~Reed Garrett}, \bibinfo{person}{Tomas
  Lozano-Perez}, {and} \bibinfo{person}{Leslie~Pack Kaelbling}.}
  \bibinfo{year}{2018}\natexlab{}.
\newblock \showarticletitle{FFRob: Leveraging symbolic planning for efficient
  task and motion planning}.
\newblock \bibinfo{journal}{{\em The International Journal of Robotics
  Research\/}} \bibinfo{volume}{37}, \bibinfo{number}{1}
  (\bibinfo{year}{2018}), \bibinfo{pages}{104--136}.
\newblock


\bibitem[\protect\citeauthoryear{Girard}{Girard}{2012}]%
        {Girard12}
\bibfield{author}{\bibinfo{person}{Antoine Girard}.}
  \bibinfo{year}{2012}\natexlab{}.
\newblock \showarticletitle{Controller synthesis for safety and reachability
  via approximate bisimulation}.
\newblock \bibinfo{journal}{{\em Automatica\/}} \bibinfo{volume}{48},
  \bibinfo{number}{5} (\bibinfo{year}{2012}), \bibinfo{pages}{947--953}.
\newblock


\bibitem[\protect\citeauthoryear{Gurobi~Optimization}{Gurobi~Optimization}{2020}]%
        {gurobi2020gurobi}
\bibfield{author}{\bibinfo{person}{Incorporate Gurobi~Optimization}.}
  \bibinfo{year}{2020}\natexlab{}.
\newblock \showarticletitle{Gurobi optimizer reference manual}.
\newblock \bibinfo{journal}{{\em URL http://www. gurobi. com\/}}
  (\bibinfo{year}{2020}).
\newblock


\bibitem[\protect\citeauthoryear{Helmert}{Helmert}{2002}]%
        {helmert2002decidability}
\bibfield{author}{\bibinfo{person}{Malte Helmert}.}
  \bibinfo{year}{2002}\natexlab{}.
\newblock \showarticletitle{Decidability and Undecidability Results for
  Planning with Numerical State Variables.}. In \bibinfo{booktitle}{{\em
  AIPS}}. \bibinfo{pages}{44--53}.
\newblock


\bibitem[\protect\citeauthoryear{Henzinger, Kopke, Puri, and Varaiya}{Henzinger
  et~al\mbox{.}}{1998}]%
        {henzinger1998s}
\bibfield{author}{\bibinfo{person}{Thomas~A Henzinger},
  \bibinfo{person}{Peter~W Kopke}, \bibinfo{person}{Anuj Puri}, {and}
  \bibinfo{person}{Pravin Varaiya}.} \bibinfo{year}{1998}\natexlab{}.
\newblock \showarticletitle{What's decidable about hybrid automata?}
\newblock \bibinfo{journal}{{\em Journal of computer and system sciences\/}}
  \bibinfo{volume}{57}, \bibinfo{number}{1} (\bibinfo{year}{1998}),
  \bibinfo{pages}{94--124}.
\newblock


\bibitem[\protect\citeauthoryear{Herbert, Chen, Han, Bansal, Fisac, and
  Tomlin}{Herbert et~al\mbox{.}}{2017}]%
        {herbert2017fastrack}
\bibfield{author}{\bibinfo{person}{Sylvia~L Herbert}, \bibinfo{person}{Mo
  Chen}, \bibinfo{person}{SooJean Han}, \bibinfo{person}{Somil Bansal},
  \bibinfo{person}{Jaime~F Fisac}, {and} \bibinfo{person}{Claire~J Tomlin}.}
  \bibinfo{year}{2017}\natexlab{}.
\newblock \showarticletitle{{FaSTrack}: {A} modular framework for fast and
  guaranteed safe motion planning}. In \bibinfo{booktitle}{{\em 2017 IEEE 56th
  Annual Conference on Decision and Control (CDC)}}. IEEE,
  \bibinfo{pages}{1517--1522}.
\newblock


\bibitem[\protect\citeauthoryear{Hoffmann}{Hoffmann}{2003}]%
        {hoffmann2003metric}
\bibfield{author}{\bibinfo{person}{J{\"o}rg Hoffmann}.}
  \bibinfo{year}{2003}\natexlab{}.
\newblock \showarticletitle{The Metric-FF Planning System:
  Translating``Ignoring Delete Lists''to Numeric State Variables}.
\newblock \bibinfo{journal}{{\em Journal of artificial intelligence
  research\/}}  \bibinfo{volume}{20} (\bibinfo{year}{2003}),
  \bibinfo{pages}{291--341}.
\newblock


\bibitem[\protect\citeauthoryear{Hoffmann and Nebel}{Hoffmann and
  Nebel}{2001}]%
        {hoffmann2001ff}
\bibfield{author}{\bibinfo{person}{J{\"o}rg Hoffmann} {and}
  \bibinfo{person}{Bernhard Nebel}.} \bibinfo{year}{2001}\natexlab{}.
\newblock \showarticletitle{The FF planning system: Fast plan generation
  through heuristic search}.
\newblock \bibinfo{journal}{{\em Journal of Artificial Intelligence
  Research\/}}  \bibinfo{volume}{14} (\bibinfo{year}{2001}),
  \bibinfo{pages}{253--302}.
\newblock


\bibitem[\protect\citeauthoryear{Hofmann and Williams}{Hofmann and
  Williams}{2006}]%
        {hofmann2006robust}
\bibfield{author}{\bibinfo{person}{Andreas~G Hofmann} {and}
  \bibinfo{person}{Brian~Charles Williams}.} \bibinfo{year}{2006}\natexlab{}.
\newblock \showarticletitle{Robust Execution of Temporally Flexible Plans for
  Bipedal Walking Devices.}. In \bibinfo{booktitle}{{\em ICAPS}}.
  \bibinfo{pages}{386--389}.
\newblock


\bibitem[\protect\citeauthoryear{Janson, Schmerling, Clark, and Pavone}{Janson
  et~al\mbox{.}}{2015}]%
        {janson2015fast}
\bibfield{author}{\bibinfo{person}{Lucas Janson}, \bibinfo{person}{Edward
  Schmerling}, \bibinfo{person}{Ashley Clark}, {and} \bibinfo{person}{Marco
  Pavone}.} \bibinfo{year}{2015}\natexlab{}.
\newblock \showarticletitle{Fast marching tree: A fast marching sampling-based
  method for optimal motion planning in many dimensions}.
\newblock \bibinfo{journal}{{\em International Journal of Robotics Research\/}}
  \bibinfo{volume}{34}, \bibinfo{number}{7} (\bibinfo{year}{2015}),
  \bibinfo{pages}{883--921}.
\newblock


\bibitem[\protect\citeauthoryear{Kaelbling and Lozano-P{\'e}rez}{Kaelbling and
  Lozano-P{\'e}rez}{2011}]%
        {kaelbling2011hierarchical}
\bibfield{author}{\bibinfo{person}{Leslie~Pack Kaelbling} {and}
  \bibinfo{person}{Tom{\'a}s Lozano-P{\'e}rez}.}
  \bibinfo{year}{2011}\natexlab{}.
\newblock \showarticletitle{Hierarchical task and motion planning in the now}.
  In \bibinfo{booktitle}{{\em 2011 IEEE International Conference on Robotics
  and Automation}}. IEEE, \bibinfo{pages}{1470--1477}.
\newblock


\bibitem[\protect\citeauthoryear{Kavraki, Svestka, Latombe, and
  Overmars}{Kavraki et~al\mbox{.}}{1996}]%
        {kavraki1994probabilistic}
\bibfield{author}{\bibinfo{person}{Lydia~E Kavraki}, \bibinfo{person}{Petr
  Svestka}, \bibinfo{person}{J-C Latombe}, {and} \bibinfo{person}{Mark~H
  Overmars}.} \bibinfo{year}{1996}\natexlab{}.
\newblock \showarticletitle{Probabilistic roadmaps for path planning in
  high-dimensional configuration spaces}.
\newblock \bibinfo{journal}{{\em IEEE Transactions on Robotics and
  Automation\/}} \bibinfo{volume}{12}, \bibinfo{number}{4}
  (\bibinfo{year}{1996}), \bibinfo{pages}{566--580}.
\newblock


\bibitem[\protect\citeauthoryear{Kloetzer and Belta}{Kloetzer and
  Belta}{2008}]%
        {KloetzerB08}
\bibfield{author}{\bibinfo{person}{Marius Kloetzer} {and}
  \bibinfo{person}{Calin Belta}.} \bibinfo{year}{2008}\natexlab{}.
\newblock \showarticletitle{A Fully Automated Framework for Control of Linear
  Systems from Temporal Logic Specifications}.
\newblock \bibinfo{journal}{{\it {IEEE} Trans. Automat. Control}}
  \bibinfo{volume}{53}, \bibinfo{number}{1} (\bibinfo{year}{2008}),
  \bibinfo{pages}{287--297}.
\newblock


\bibitem[\protect\citeauthoryear{Kress-Gazit, Fainekos, and Pappas}{Kress-Gazit
  et~al\mbox{.}}{2009}]%
        {LTLMOP2}
\bibfield{author}{\bibinfo{person}{Hadas Kress-Gazit},
  \bibinfo{person}{Gerogios~E. Fainekos}, {and} \bibinfo{person}{George~J.
  Pappas}.} \bibinfo{year}{2009}\natexlab{}.
\newblock \showarticletitle{Temporal Logic based Reactive Mission and Motion
  Planning}.
\newblock \bibinfo{journal}{{\em IEEE Transactions on Robotics\/}}
  \bibinfo{volume}{25}, \bibinfo{number}{6} (\bibinfo{year}{2009}),
  \bibinfo{pages}{1370{\textendash}1381}.
\newblock


\bibitem[\protect\citeauthoryear{Kuffner and LaValle}{Kuffner and
  LaValle}{2000}]%
        {kuffner2000rrt}
\bibfield{author}{\bibinfo{person}{James~J Kuffner} {and}
  \bibinfo{person}{Steven~M LaValle}.} \bibinfo{year}{2000}\natexlab{}.
\newblock \showarticletitle{{RRT}-{C}onnect: An efficient approach to
  single-query path planning}. In \bibinfo{booktitle}{{\em IEEE International
  Conference on Robotics and Automation}}, Vol.~\bibinfo{volume}{2}. IEEE,
  \bibinfo{pages}{995--1001}.
\newblock


\bibitem[\protect\citeauthoryear{Lagriffoul, Dantam, Garrett, Akbari,
  Srivastava, and Kavraki}{Lagriffoul et~al\mbox{.}}{2018}]%
        {lagriffoul2018platform}
\bibfield{author}{\bibinfo{person}{Fabien Lagriffoul}, \bibinfo{person}{Neil~T
  Dantam}, \bibinfo{person}{Caelan Garrett}, \bibinfo{person}{Aliakbar Akbari},
  \bibinfo{person}{Siddharth Srivastava}, {and} \bibinfo{person}{Lydia~E
  Kavraki}.} \bibinfo{year}{2018}\natexlab{}.
\newblock \showarticletitle{Platform-independent benchmarks for task and motion
  planning}.
\newblock \bibinfo{journal}{{\em IEEE Robotics and Automation Letters\/}}
  \bibinfo{volume}{3}, \bibinfo{number}{4} (\bibinfo{year}{2018}),
  \bibinfo{pages}{3765--3772}.
\newblock


\bibitem[\protect\citeauthoryear{Lahijanian, Kavraki, and Vardi}{Lahijanian
  et~al\mbox{.}}{2014}]%
        {lahijanian2014sampling}
\bibfield{author}{\bibinfo{person}{Morteza Lahijanian},
  \bibinfo{person}{Lydia~E Kavraki}, {and} \bibinfo{person}{Moshe~Y Vardi}.}
  \bibinfo{year}{2014}\natexlab{}.
\newblock \showarticletitle{A sampling-based strategy planner for
  nondeterministic hybrid systems}. In \bibinfo{booktitle}{{\em 2014 IEEE
  International Conference on Robotics and Automation (ICRA)}}. IEEE,
  \bibinfo{pages}{3005--3012}.
\newblock


\bibitem[\protect\citeauthoryear{Laurenti, Lahijanian, Abate, Cardelli, and
  Kwiatkowska}{Laurenti et~al\mbox{.}}{2020}]%
        {laurenti2020formal}
\bibfield{author}{\bibinfo{person}{Luca Laurenti}, \bibinfo{person}{Morteza
  Lahijanian}, \bibinfo{person}{Alessandro Abate}, \bibinfo{person}{Luca
  Cardelli}, {and} \bibinfo{person}{Marta Kwiatkowska}.}
  \bibinfo{year}{2020}\natexlab{}.
\newblock \showarticletitle{Formal and efficient synthesis for continuous-time
  linear stochastic hybrid processes}.
\newblock \bibinfo{journal}{{\it IEEE Trans. Automat. Control}}
  (\bibinfo{year}{2020}).
\newblock


\bibitem[\protect\citeauthoryear{Li and Williams}{Li and Williams}{2008}]%
        {li2008generative}
\bibfield{author}{\bibinfo{person}{Hui~X Li} {and} \bibinfo{person}{Brian~C
  Williams}.} \bibinfo{year}{2008}\natexlab{}.
\newblock \showarticletitle{Generative Planning for Hybrid Systems Based on
  Flow Tubes.}. In \bibinfo{booktitle}{{\em ICAPS}}. \bibinfo{pages}{206--213}.
\newblock


\bibitem[\protect\citeauthoryear{Lynch, Segala, Vaandrager, and Weinberg}{Lynch
  et~al\mbox{.}}{1995}]%
        {lynch1995hybrid}
\bibfield{author}{\bibinfo{person}{Nancy Lynch}, \bibinfo{person}{Roberto
  Segala}, \bibinfo{person}{Frits Vaandrager}, {and} \bibinfo{person}{Henri~B
  Weinberg}.} \bibinfo{year}{1995}\natexlab{}.
\newblock \showarticletitle{Hybrid i/o automata}. In \bibinfo{booktitle}{{\em
  International Hybrid Systems Workshop}}. Springer, \bibinfo{pages}{496--510}.
\newblock


\bibitem[\protect\citeauthoryear{Maler and Nickovic}{Maler and
  Nickovic}{2004}]%
        {maler2004monitoring}
\bibfield{author}{\bibinfo{person}{Oded Maler} {and} \bibinfo{person}{Dejan
  Nickovic}.} \bibinfo{year}{2004}\natexlab{}.
\newblock \showarticletitle{Monitoring temporal properties of continuous
  signals}.
\newblock In \bibinfo{booktitle}{{\em Formal Techniques, Modelling and Analysis
  of Timed and Fault-Tolerant Systems}}. \bibinfo{publisher}{Springer},
  \bibinfo{pages}{152--166}.
\newblock


\bibitem[\protect\citeauthoryear{Mallik, Schmuck, Soudjani, and
  Majumdar}{Mallik et~al\mbox{.}}{2018}]%
        {MajumdarMS16}
\bibfield{author}{\bibinfo{person}{Kaushik Mallik},
  \bibinfo{person}{Anne-Kathrin Schmuck}, \bibinfo{person}{Sadegh Soudjani},
  {and} \bibinfo{person}{Rupak Majumdar}.} \bibinfo{year}{2018}\natexlab{}.
\newblock \showarticletitle{Compositional synthesis of finite-state
  abstractions}.
\newblock \bibinfo{journal}{{\it IEEE Trans. Automat. Control}}
  \bibinfo{volume}{64}, \bibinfo{number}{6} (\bibinfo{year}{2018}),
  \bibinfo{pages}{2629--2636}.
\newblock


\bibitem[\protect\citeauthoryear{Mouelhi, Girard, and G\"{o}ssler}{Mouelhi
  et~al\mbox{.}}{2013}]%
        {Cosyma}
\bibfield{author}{\bibinfo{person}{Sebti Mouelhi}, \bibinfo{person}{Antoine
  Girard}, {and} \bibinfo{person}{Gregor G\"{o}ssler}.}
  \bibinfo{year}{2013}\natexlab{}.
\newblock \showarticletitle{{CoSyMA}: A Tool for Controller Synthesis Using
  Multi-scale Abstractions}. In \bibinfo{booktitle}{{\em International
  Conference on Hybrid Systems: Computation and Control}}.
  \bibinfo{publisher}{ACM}, \bibinfo{pages}{83--88}.
\newblock


\bibitem[\protect\citeauthoryear{Plaku, Kavraki, and Vardi}{Plaku
  et~al\mbox{.}}{2013}]%
        {plaku2013falsification}
\bibfield{author}{\bibinfo{person}{Erion Plaku}, \bibinfo{person}{Lydia~E
  Kavraki}, {and} \bibinfo{person}{Moshe~Y Vardi}.}
  \bibinfo{year}{2013}\natexlab{}.
\newblock \showarticletitle{Falsification of LTL safety properties in hybrid
  systems}.
\newblock \bibinfo{journal}{{\em International Journal on Software Tools for
  Technology Transfer\/}} \bibinfo{volume}{15}, \bibinfo{number}{4}
  (\bibinfo{year}{2013}), \bibinfo{pages}{305--320}.
\newblock


\bibitem[\protect\citeauthoryear{Raman, Donz{\'e}, Sadigh, Murray, and
  Seshia}{Raman et~al\mbox{.}}{2015}]%
        {raman2015reactive}
\bibfield{author}{\bibinfo{person}{Vasumathi Raman}, \bibinfo{person}{Alexandre
  Donz{\'e}}, \bibinfo{person}{Dorsa Sadigh}, \bibinfo{person}{Richard~M
  Murray}, {and} \bibinfo{person}{Sanjit~A Seshia}.}
  \bibinfo{year}{2015}\natexlab{}.
\newblock \showarticletitle{Reactive synthesis from signal temporal logic
  specifications}. In \bibinfo{booktitle}{{\em Proceedings of the 18th
  international conference on hybrid systems: Computation and control}}.
  \bibinfo{pages}{239--248}.
\newblock


\bibitem[\protect\citeauthoryear{Roy, Tabuada, and Majumdar}{Roy
  et~al\mbox{.}}{2011}]%
        {pessoa}
\bibfield{author}{\bibinfo{person}{Pritam Roy}, \bibinfo{person}{Paulo
  Tabuada}, {and} \bibinfo{person}{Rupak Majumdar}.}
  \bibinfo{year}{2011}\natexlab{}.
\newblock \showarticletitle{Pessoa 2.0: A Controller Synthesis Tool for
  Cyber-physical Systems}. In \bibinfo{booktitle}{{\em International Conference
  on Hybrid Systems: Computation and Control}}. \bibinfo{publisher}{ACM},
  \bibinfo{pages}{315--316}.
\newblock


\bibitem[\protect\citeauthoryear{Rungger and Zamani}{Rungger and
  Zamani}{2016}]%
        {rungger2016scots}
\bibfield{author}{\bibinfo{person}{Matthias Rungger} {and}
  \bibinfo{person}{Majid Zamani}.} \bibinfo{year}{2016}\natexlab{}.
\newblock \showarticletitle{SCOTS: A tool for the synthesis of symbolic
  controllers}. In \bibinfo{booktitle}{{\em Proceedings of the 19th
  international conference on hybrid systems: Computation and control}}.
  \bibinfo{pages}{99--104}.
\newblock


\bibitem[\protect\citeauthoryear{Tabuada}{Tabuada}{2009}]%
        {paulobook}
\bibfield{author}{\bibinfo{person}{Paulo Tabuada}.}
  \bibinfo{year}{2009}\natexlab{}.
\newblock \bibinfo{booktitle}{{\em Verification and Control of Hybrid Systems -
  {A} Symbolic Approach}}.
\newblock \bibinfo{publisher}{Springer}.
\newblock


\bibitem[\protect\citeauthoryear{Tabuada and Pappas}{Tabuada and
  Pappas}{2006}]%
        {TabuadaP06}
\bibfield{author}{\bibinfo{person}{Paulo Tabuada} {and}
  \bibinfo{person}{George~J. Pappas}.} \bibinfo{year}{2006}\natexlab{}.
\newblock \showarticletitle{Linear Time Logic Control of Discrete-Time Linear
  Systems}.
\newblock \bibinfo{journal}{{\it {IEEE} Trans. Automat. Control}}
  \bibinfo{volume}{51}, \bibinfo{number}{12} (\bibinfo{year}{2006}),
  \bibinfo{pages}{1862--1877}.
\newblock


\bibitem[\protect\citeauthoryear{Vaskov, Kousik, Larson, Bu, Ward, Worrall,
  Johnson-Roberson, and Vasudevan}{Vaskov et~al\mbox{.}}{2019}]%
        {vaskov2019towards}
\bibfield{author}{\bibinfo{person}{Sean Vaskov}, \bibinfo{person}{Shreyas
  Kousik}, \bibinfo{person}{Hannah Larson}, \bibinfo{person}{Fan Bu},
  \bibinfo{person}{James Ward}, \bibinfo{person}{Stewart Worrall},
  \bibinfo{person}{Matthew Johnson-Roberson}, {and} \bibinfo{person}{Ram
  Vasudevan}.} \bibinfo{year}{2019}\natexlab{}.
\newblock \showarticletitle{Towards provably not-at-fault control of autonomous
  robots in arbitrary dynamic environments}.
\newblock \bibinfo{journal}{{\em arXiv preprint arXiv:1902.02851\/}}
  (\bibinfo{year}{2019}).
\newblock


\bibitem[\protect\citeauthoryear{Vidal, Schaffert, Shakernia, Lygeros, and
  Sastry}{Vidal et~al\mbox{.}}{2001}]%
        {vidal2001decidable}
\bibfield{author}{\bibinfo{person}{Ren{\'e} Vidal}, \bibinfo{person}{Shawn
  Schaffert}, \bibinfo{person}{Omid Shakernia}, \bibinfo{person}{John Lygeros},
  {and} \bibinfo{person}{Shankar Sastry}.} \bibinfo{year}{2001}\natexlab{}.
\newblock \showarticletitle{Decidable and semi-decidable controller synthesis
  for classes of discrete time hybrid systems}. In \bibinfo{booktitle}{{\em
  Proceedings of the 40th IEEE Conference on Decision and Control (Cat. No.
  01CH37228)}}, Vol.~\bibinfo{volume}{2}. IEEE, \bibinfo{pages}{1243--1248}.
\newblock


\bibitem[\protect\citeauthoryear{Wong, Finucane, and Kress{-}Gazit}{Wong
  et~al\mbox{.}}{2013}]%
        {LTLMOP}
\bibfield{author}{\bibinfo{person}{Kai~Weng Wong}, \bibinfo{person}{Cameron
  Finucane}, {and} \bibinfo{person}{Hadas Kress{-}Gazit}.}
  \bibinfo{year}{2013}\natexlab{}.
\newblock \showarticletitle{Provably-correct robot control with {LTLMoP},
  {OMPL} and {ROS}}. In \bibinfo{booktitle}{{\em {IEEE/RSJ} International
  Conference on Intelligent Robots and Systems}}. \bibinfo{pages}{2073}.
\newblock


\bibitem[\protect\citeauthoryear{Wongpiromsarn, Topcu, and
  Murray}{Wongpiromsarn et~al\mbox{.}}{2012}]%
        {WongpiromsarnTM12}
\bibfield{author}{\bibinfo{person}{Tichakorn Wongpiromsarn},
  \bibinfo{person}{Ufuk Topcu}, {and} \bibinfo{person}{Richard~M. Murray}.}
  \bibinfo{year}{2012}\natexlab{}.
\newblock \showarticletitle{Receding Horizon Temporal Logic Planning}.
\newblock \bibinfo{journal}{{\it {IEEE} Trans. Automat. Control}}
  \bibinfo{volume}{57}, \bibinfo{number}{11} (\bibinfo{year}{2012}),
  \bibinfo{pages}{2817--2830}.
\newblock


\bibitem[\protect\citeauthoryear{Wongpiromsarn, Topcu, Ozay, Xu, and
  Murray}{Wongpiromsarn et~al\mbox{.}}{2011a}]%
        {Tulip-short}
\bibfield{author}{\bibinfo{person}{Tichakorn Wongpiromsarn},
  \bibinfo{person}{Ufuk Topcu}, \bibinfo{person}{Necmiye Ozay},
  \bibinfo{person}{Huan Xu}, {and} \bibinfo{person}{Richard~M. Murray}.}
  \bibinfo{year}{2011}\natexlab{a}.
\newblock \showarticletitle{TuLiP: A Software Toolbox for Receding Horizon
  Temporal Logic Planning}. In \bibinfo{booktitle}{{\em International
  Conference on Hybrid Systems: Computation and Control}}.
  \bibinfo{publisher}{ACM}, \bibinfo{pages}{313--314}.
\newblock


\bibitem[\protect\citeauthoryear{Wongpiromsarn, Topcu, Ozay, Xu, and
  Murray}{Wongpiromsarn et~al\mbox{.}}{2011b}]%
        {wongpiromsarn2011tulip}
\bibfield{author}{\bibinfo{person}{Tichakorn Wongpiromsarn},
  \bibinfo{person}{Ufuk Topcu}, \bibinfo{person}{Necmiye Ozay},
  \bibinfo{person}{Huan Xu}, {and} \bibinfo{person}{Richard~M Murray}.}
  \bibinfo{year}{2011}\natexlab{b}.
\newblock \showarticletitle{TuLiP: a software toolbox for receding horizon
  temporal logic planning}. In \bibinfo{booktitle}{{\em Proceedings of the 14th
  international conference on Hybrid systems: computation and control}}.
  \bibinfo{pages}{313--314}.
\newblock


\end{thebibliography}
